\documentclass{article}

\usepackage{microtype}
\usepackage{graphicx}
\usepackage{subcaption}
\usepackage{booktabs} 
\usepackage{multirow}
\usepackage{pifont}
\usepackage{makecell}
\usepackage{caption}
\usepackage{subcaption}
\usepackage[T1]{fontenc}
\usepackage{amssymb}
\usepackage{rotating}
\usepackage{enumitem}
\usepackage[skins,breakable]{tcolorbox}

\usepackage{hyperref}

\usepackage[preprint]{icml2026}

\usepackage{amsmath}
\usepackage{amssymb}
\usepackage{mathtools}
\usepackage{amsthm}
\usepackage[table]{xcolor}

\usepackage[capitalize,noabbrev]{cleveref}

\theoremstyle{plain}
\newtheorem{theorem}{Theorem}[section]
\newtheorem{proposition}[theorem]{Proposition}

\theoremstyle{definition}

\theoremstyle{remark}

\usepackage[textsize=tiny]{todonotes}

\icmltitlerunning{Reasoning as State Transition: A Representational Analysis of Reasoning Evolution in LLMs}

\begin{document}

\twocolumn[
  \icmltitle{Reasoning as State Transition: A Representational Analysis of Reasoning Evolution in Large Language Models}



  \icmlsetsymbol{equal}{*}

  \begin{icmlauthorlist}
    \icmlauthor{Siyuan Zhang}{cs}
    \icmlauthor{Jialian Li}{cs}
    \icmlauthor{Yichi Zhang}{cs}
    \icmlauthor{Xiao Yang}{cs}
    \icmlauthor{Yinpeng Dong}{ailab}
    \icmlauthor{Hang Su}{cs}
  \end{icmlauthorlist}

  \icmlaffiliation{cs}{Dept. of Comp. Sci. and Tech., Institute for AI, Tsinghua-Bosch Joint ML Center, THBI Lab, BNRist Center, Tsinghua University, Beijing 100084, China}
  \icmlaffiliation{ailab}{College of AI, Tsinghua University, Beijing 100084, China}

  \icmlcorrespondingauthor{Yinpeng Dong}{dongyinpeng@mail.tsinghua.edu.cn}
  \icmlcorrespondingauthor{Hang Su}{suhangss@mail.tsinghua.edu.cn}

  \icmlkeywords{Large Language Model, Reasoning, Representation, ICML}

  \vskip 0.3in
]



\printAffiliationsAndNotice{}  

\begin{abstract}
  Large Language Models have achieved remarkable performance on reasoning tasks, motivating research into how this ability evolves during training. Prior work has primarily analyzed this evolution via explicit generation outcomes, treating the reasoning process as a black box and obscuring internal changes. To address this opacity, we introduce a representational perspective to investigate the dynamics of the model's internal states. Through comprehensive experiments across models at various training stages, we discover that post-training yields only limited improvement in static initial representation quality. Furthermore, we reveal that, distinct from non-reasoning tasks, reasoning involves a significant continuous distributional shift in representations during generation. Comparative analysis indicates that post-training empowers models to drive this transition toward a better distribution for task solving. To clarify the relationship between internal states and external outputs, statistical analysis confirms a high correlation between generation correctness and the final representations; while counterfactual experiments identify the semantics of the generated tokens, rather than additional computation during inference or intrinsic parameter differences, as the dominant driver of the transition. Collectively, we offer a novel understanding of the reasoning process and the effect of training on reasoning enhancement, providing valuable insights for future model analysis and optimization.
  \vspace{-3ex}
\end{abstract}

\section{Introduction}
\vspace{-1ex}
Large Language Models (LLMs) \cite{jaech2024openai, guo2025deepseek} have demonstrated remarkable performance on reasoning tasks requiring multi-hop analysis and complex calculations \cite{yue2025does, wu2025invisible}. Through targeted post-training techniques, specifically Reinforcement Learning (RL) \cite{shao2024deepseekmath, zeng2025simplerl, zheng2025group} and distillation \cite{guo2025deepseek, openr1, tian2025deepdistill}, these models acquire the capacity to generate long Chain-of-Thought (CoT) sequences \cite{wei2022chain}. This capability enables models to effectively decompose and solve problems that remain challenging for base models \cite{li2025llms, wu2025ctrls}. And such advancements have motivated research into the underlying mechanisms of how reasoning evolves during training \cite{zhao2025echo, matsutani2025rl}.

Prior research has primarily analyzed this evolution from the perspective of the explicit \textit{generation} process, focusing on the improvement in response accuracy \cite{yue2025does, shao2025spurious, zhang2025interplay} or the emergence of special CoT patterns \cite{li2025llms, matsutani2025rl, wang2025beyond}. However, by relying solely on external behaviors, these approaches treat reasoning as a black box, overlooking the dynamics of the model's internal states. Consequently, they fail to provide a deeper and mechanistic explanation of how reasoning functions and how reasoning ability fundamentally develops through training. 

\begin{figure*}[ht]
  \begin{center}
    \centerline{\includegraphics[width=\textwidth]{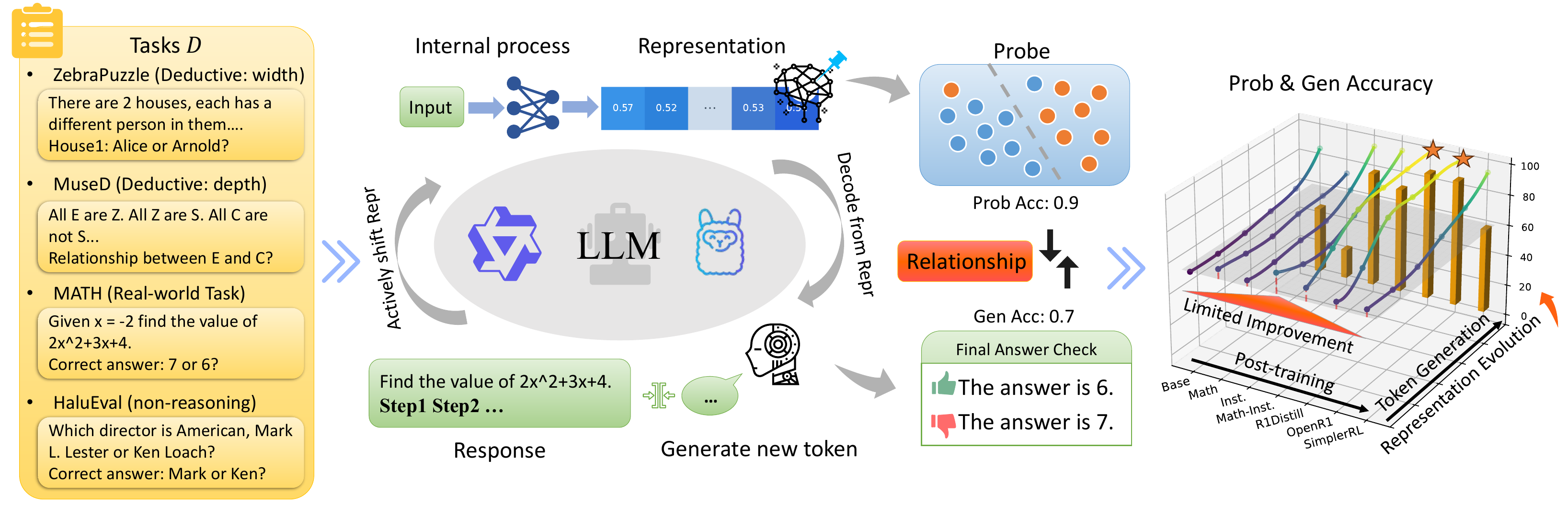}}
    \caption{\textbf{Overview of the dual-perspective analysis framework.} We conduct a comprehensive study across four representative tasks, measuring representation quality via probing alongside explicit generation accuracy. Our results indicate that post-training significantly improves generation accuracy but yields limited enhancement to static initial representation quality. Furthermore, tracking representation dynamics reveals that the reasoning process involves a significant distributional shift, where post-training empowers models to drive representations toward a higher-quality state. Finally, we investigate the relationship between explicit generation and implicit representation, quantifying their alignment and identifying the dominant factor driving this state transition.} 
    \label{fig:intro}
    \vspace{-4ex}
  \end{center}
\end{figure*}

To address this limitation and enhance interpretability, we extend our investigation beyond explicit generation to the model's latent \textit{representations}, which serve as a proxy for its internal states. We quantify representation quality using the probing technique \cite{ye2024physics}, which measures how accurately the correct answer can be predicted directly from the high-dimensional representation vector, thereby reflecting whether the model internally encodes the solution to a task. As illustrated in \cref{fig:intro}, our analysis framework examines both the static \textit{initial} quality prior to generation and the dynamic quality variation throughout the generation process. Comprehensive experiments on models at different training stages, spanning diverse representative tasks, LLM families, and model sizes, reveal several surprising yet universal findings, as detailed below.

As a foundational starting point for the reasoning process, we investigate the evolution of initial representation quality during post-training. We find that reasoning ability is already acquired during the pre-training phase. LLMs exhibit significant representation quality on reasoning tasks, with probing accuracy far exceeding random guessing. Notably, even without token generation, this initial representation quality surpasses explicit generation accuracy in over $60\%$ of tasks, indicating a strong latent reasoning ability. However, differing from the substantial gains in generation accuracy, post-training improves the initial representation quality to only a limited extent, with probing accuracy increasing by less than $5\%$ across nearly all training stages. This demonstrates that the improvement in generation accuracy does not simply stem from a better reasoning ability available immediately upon encountering a question.

Another observation is that strong reasoning models can achieve higher generation accuracy that exceeds their initial representation quality on difficult tasks. Since this cannot be explained by the negligible improvement in static initial quality, we further investigate the dynamics of the representation during generation. Our experiments reveal that, while representation quality remains relatively stable on non-reasoning tasks, reasoning tasks involve a continuous and significant representation change throughout the generation process. Post-training serves to enhance the model's capacity to drive this transition, enabling it to attain a higher \textit{final} quality after extensive CoT. Notably, although explicit generation accuracy can surpass the initial representation quality, it generally lags behind the final quality, indicating untapped potential within the model's internal states.

While we have observed internal changes during reasoning, the precise relationship between explicit generation and implicit representations remains underexplored. First, we statistically investigate their alignment, finding a high correlation between generation correctness and the final representations. Conversely, the low correlation with initial representations suggests that the model does not merely verbalize its initial latent thoughts; rather, the generation process actively constructs the solution independent of the initial state. Furthermore, the sharp divergence between initial and final representations confirms a fundamental distributional shift, and the extensive reasoning process facilitates this state transition. Second, we identify the primary driver of this transition. Counterfactual analysis demonstrates that the semantic content of the CoT is the dominant factor, whereas solely increasing computation during inference is insufficient to improve representation quality, and intrinsic parameter differences affect the shift results minimally.

Through comprehensive and in-depth analysis, our work offers a novel perspective to understand the reasoning process and the evolution of reasoning during training, elucidating the relationship between the model's internal representations and external generation. These findings provide valuable insights for further reasoning interpretability and highlight promising directions for model optimization from the viewpoint of internal states and signals.

\section{Related Work}
\textbf{Analysis of LLM reasoning evolution.} Following the success of reasoning models \cite{jaech2024openai, guo2025deepseek, team2025kimi}, extensive research has investigated how reasoning ability evolves during post-training. One line of work \cite{li2025llms, wang2025beyond, gandhi2025cognitive, vassoyan2025ignore} analyzes training data, identifying that specific patterns, such as reasoning structures and high-entropy tokens, are crucial for optimization results. Another stream of research \cite{havrilla2024teaching, yue2025does, liu2025understanding, wu2025invisible, zhang2025interplay, karan2025reasoning, matsutani2025rl} compares reasoning boundaries before and after training, revealing that while RL may not extend reasoning capability beyond the base model, distillation effectively transfers new abilities. However, these studies mainly analyze reasoning through explicit generation outcomes, treating the model as a black box and overlooking changes in the model's internal states. To address this limitation, we introduce another perspective from implicit representations to directly observe the model's latent judgments and reasoning processes. This approach uncovers meaningful findings through a distinct viewpoint, providing novel insights into both reasoning interpretability and the underlying mechanisms of reasoning evolution.

\textbf{Analysis of LLM representation.} Interpretability research suggests that internal representations encode meaningful knowledge \cite{hewitt2019structural, mohebbi2021exploring, ghandeharioun2024patchscopes, Gurnee2024language, atakishiyev2025explainability, yusupov2025internal}, enabling the prediction of specific properties via probing techniques \cite{belinkov2022probing, zhao2024explainability}. Prior studies have leveraged these representations to detect output correctness \cite{orgad2025llms, zhang2025icr, cencerrado2025no, zhang2025reasoning, liang2025clue}, uncertainty \cite{wang2025response}, and quality \cite{yusupov2025internal}. Distinct from predicting external generation performance, our analysis focuses on determining whether the model has internally discerned the correct solution for a task. While recent studies \cite{yan2024exploring, ye2024physics} have revealed that models possess an awareness of correct answers via latent reasoning \cite{yang2024large, biran2024hopping} prior to generation, we extend this investigation to complex reasoning tasks and analyze the dynamic evolution of this awareness throughout the CoT process. Most closely related to our work, recent studies \cite{kudo2024think, afzal2025knowing, wang2025chain} have also examined representation shifts during generation. However, we distinguish our contribution by analyzing how these shifts evolve across different training stages and by offering a deeper and novel analysis of the relationship between the internal representations and the external generation.

\section{Preliminaries}
In this section, we introduce the preliminaries for our experiments and analysis, detailing the formulation of core metrics and the experimental setup.

\subsection{Formulation}
We define two core metrics in our analysis: representation quality and generation accuracy. For clarity, we first formalize the LLM working process. We consider a Transformer-based model \cite{vaswani2017attention} $M_\theta$, which can be functionally decomposed into two components: a backbone function $f_\theta$ that maps the input to a latent representation and a decoding function $g_\theta$ that computes the next-token distribution from this representation. Following previous work \cite{huh2024platonic, yan2024exploring, zhang-etal-2025-exploring-generalizability}, we define the representation $c$ as the hidden state of the last token at the final layer. Given a prompt $x$, the auto-regressive generation proceeds through discrete time steps $t=1,\dots,T$. At each step $t$, the input $(x, y_{1:t})$ is first transformed into the representation $c_t:=f_\theta(x, y_{1:t}) \in \mathbb{R}^m$, where $m$ denotes the hidden dimension. We adopt the representation $c_t$ as a proxy for the model's full internal state $S_t$ at step $t$. Subsequently, the next token is sampled according to $y_{t+1} \sim g_\theta(\cdot|c_t)$. We denote the complete generation process of a sequence $y_{1:T}$ given a prompt $x$ as $y_{1:T} \sim M_\theta(\cdot|x)$.

\paragraph{Representation Quality.}
We quantify representation quality using the widely-adopted probing technique \cite{belinkov2017evaluating, belinkov2022probing, azaria2023internal, sky2024androids}. Given a dataset $\mathcal{D}^\text{repr} = \{(x^{(i)}, z^{(i)}\}$ for a $s$-class task, where the label $z \in \{0,\dots,s-1\}$, standard linear-probing \cite{hewitt2019structural, li2023inference, orgad2025llms} involves freezing the LLM backbone and training a linear classifier $h \in R^{m \times s}$ on the training split. This classifier maps the representation $c$ to a predicted probability distribution $p := \text{softmax}(h(c))$. And the classifier is optimized via cross-entropy loss:
\begin{equation}
\mathcal{L_\text{CE}}(\mathcal{D}^\text{repr}_\text{train}) = \mathbb{E}_{\{x^{(i)}, z^{(i)}\}\in \mathcal{D}^\text{repr}_\text{train}}[-\log(p(z^{(i)}|x^{(i)}))].
\end{equation}
The \textbf{representation quality} $Acc_\text{repr}(\mathcal{D}^\text{repr}_\text{test})$ is then defined as the classification accuracy on the test split:
\begin{equation}
     \resizebox{0.9\hsize}{!}{
    $\mathbb{E}_{\{x^{(i)},z^{(i)}\}\in\mathcal{D}^\text{repr}_\text{test}}[\textbf{1}(\arg\max_{z}(p(z|x^{(i)})) = z^{(i)})].$}
\end{equation}

However, standard probing is insufficient for tasks where the target label is not fixed but is conditioned on multiple variables, so we adopt the V-probing technique \cite{ye2024physics}, which encapsulates relevant variables within special tokens and trains their embeddings via a low-rank update to the embedding layer. Formally, the detailed training and calculation procedure is presented in \cref{alg:probing}. 

Under this definition, representation quality functions as a static metric for a specific task and context. Specifically, when the input comprises only the problem description $x$ prior to any explicit token generation, the measured accuracy constitutes the \textbf{initial representation quality}, which characterizes the model's capacity for pure latent reasoning \cite{yang2024large} derived from a single forward pass.

Although V-probing is nearly-linear, we posit that it remains a faithful indicator of the model's intrinsic representation quality. In Appendix \ref{sec:probing_proof}, we theoretically analyzed that the limited budget of trainable parameters is insufficient to independently achieve high probing accuracy on tasks of sufficient complexity. Empirical results in \cref{tab:baseline} verify this conclusion, showing that V-probing on a randomly initialized model fails to outperform majority guessing. The substantial performance gap between the pre-trained and randomly initialized models confirms that the observed accuracy derives primarily from the knowledge encoded within the model's existing parameters rather than the probe's capacity to fully fit the data distribution.

\paragraph{Generation Accuracy.}
\textbf{Generation accuracy} quantifies the model's ability to explicitly generate correct solutions. We evaluate this metric $Acc_\text{gen}(\mathcal{D}^\text{gen}_\text{test})$ on a test set $\mathcal{D}^\text{gen}_\text{test}$ containing identical problem instances to those in the probing dataset $\mathcal{D}^\text{repr}_\text{test}$:
\begin{equation}
    \resizebox{0.9\hsize}{!}{
    $\mathbb{E}_{\{x^{(i)}, z^{(i)}\}\in \mathcal{D}^\text{gen}_\text{test}, y^{(i)}_{1:T} \sim M_\theta(·|x^{(i)})}[\textbf{1}(\text{Ext}(x^{(i)}+y_{1:T}^{(i)}) = z^{(i)})],$}
\end{equation}
where $\text{Ext}(\cdot)$ denotes a function that extracts the final answer from the model response and maps it to the target label. In contrast to the static nature of representation quality, generation accuracy reflects the dynamic outcome of the entire auto-regressive reasoning process for a given task.

\subsection{Experimental Setup}
We select four representative tasks to comprehensively analyze reasoning ability. We adopt the deductive tasks ZebraPuzzle \cite{stojanovski2025reasoninggymreasoningenvironments} and MuseD \cite{li2024boostingdeductivereasoningstep} to evaluate reasoning breadth and depth, respectively, following \citet{allen2025physics}. \textbf{ZebraPuzzle} presents constraint satisfaction problems requiring the simultaneous resolution of multiple relational dependencies, while \textbf{MuseD} demands long-range, multi-hop deductive reductions. Both synthetic tasks are stratified into three difficulty levels (Low, Med, High). In addition to the deductive tasks, we select \textbf{MATH} \cite{hendrycksmath2021} to evaluate unstructured real-world reasoning. Finally, \textbf{HaluEval} \cite{li2023halueval} serves as a control task for factuality, with minimal reasoning requirements. To satisfy the probing constraints, we construct True-False (TF) and Multiple-Choice (MC) variants for each task. Detailed specifications are provided in Appendix \ref{sec:app task}.

For V-probing, we employ LoRA \cite{lora} training with the rank $r=4$, the learning rate $\gamma=1e^{-4}$, and the batch size $b=32$. For generation, we adopt a context window of $32,768$ tokens to accommodate long CoT, consistent with recent evaluations \cite{guo2025deepseek, qwq32b}, and utilize a rule-based function for answer extraction $\text{Ext}(\cdot)$. More implementation details are provided in Appendix \ref{sec:app probing training settings} and Appendix \ref{sec:app gen}.

\begin{figure}[ht]
  \begin{center}
    \centerline{\includegraphics[width=\columnwidth]{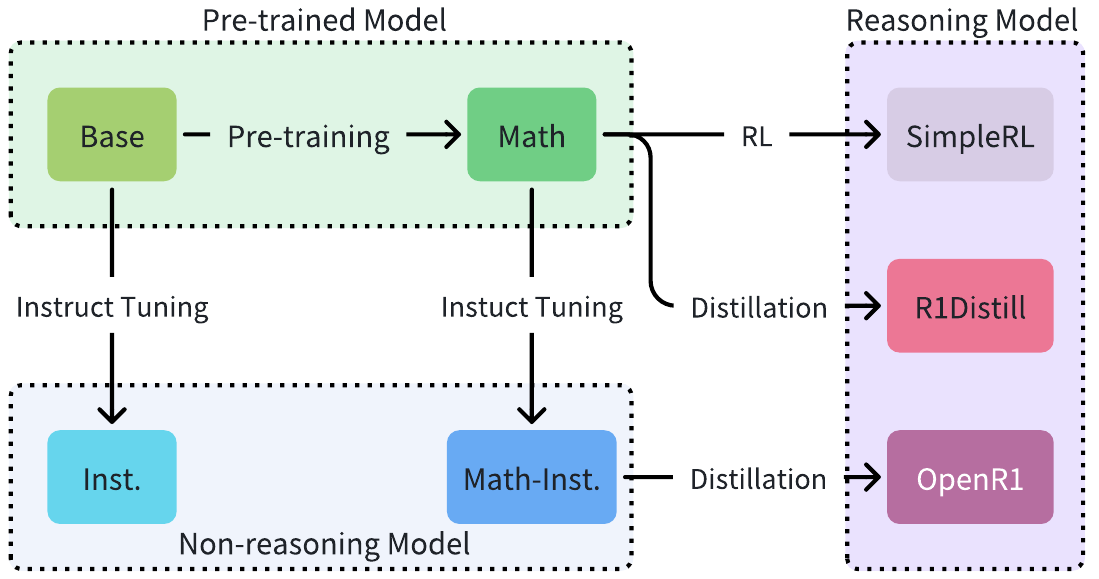}}
    \caption{
      \textbf{Training relationships among the Qwen2.5-7B series.}
    }
    \label{fig:relationship}
    \vspace{-4ex}
  \end{center}
\end{figure}

We conduct our primary analysis on the Qwen2.5-7B series \cite{qwen2.5, yang2024qwen25mathtechnicalreportmathematical}, evaluating seven models across distinct training stages. As illustrated in \cref{fig:relationship}, we categorize these models into three classes: Pre-trained models (Base, Math); Non-reasoning models with standard instruct-tuning (Instruct, Math-Instruct); and Reasoning models developed via distillation or RL (R1Distill \cite{guo2025deepseek}, OpenR1 \cite{openr1}, and SimpleRL \cite{zeng2025simplerl}). Detailed model specifications are provided in Appendix \ref{sec:app model}. To verify the universality of our findings across different model families and parameter scales, we replicate representative experiments on the Llama3.1-8B series \cite{dubey2024llama} and the smaller Qwen2.5-1.5B series, with these results detailed in Appendix \ref{sec:app result}.

\section{Representational Analysis of Reasoning Evolution}
Having established the metrics for representation quality and generation accuracy, we now present our primary experimental findings in this section.

\subsection{Development of Initial Representation Quality}
\label{sec:main exp1}
To analyze the evolution of reasoning ability in LLMs, we first examine the development of initial representation quality across various training stages.

\begin{table}[ht]
  \caption{\textbf{Probing accuracy} (\%) of the Pre-trained (PT) model across all tasks, compared against majority guessing and randomly initialized model probing baselines.}
  \label{tab:baseline}
  \begin{center}
    \begin{small}
      \begin{sc}
      \resizebox{\columnwidth}{!}{
            \begin{tabular}{l|ccc|ccc}
            \toprule[1.5px]
             & \multicolumn{3}{c|}{Zebra-TF} & \multicolumn{3}{c}{Zebra-MC} \\
            & High & Med & Low & High & Med & Low \\ \midrule
            Majority guessing & 50.00 & 50.00 & 50.00 & 22.22 & 28.57 & 40.00 \\
            \rowcolor{blue!10} \textbf{PT model probing} & \textbf{60.25} & \textbf{63.65} & \textbf{80.65} & \textbf{37.00} & \textbf{49.55} & \textbf{78.90} \\ 
            Random model probing & 52.20 & 53.80 & 51.20 & 25.00 & 30.95 & 43.25 \\ \midrule
            & \multicolumn{3}{c|}{MuseD-TF} & \multicolumn{3}{c}{MuseD-MC} \\
            & High & Med & Low & High & Med & Low \\ \midrule
            Majority guessing & 50.00 & 50.00 & 50.00 & 25.00 & 25.00 & 25.00 \\
            \rowcolor{blue!10} \textbf{PT model probing} & \textbf{92.44} & \textbf{100.00} & \textbf{100.00} & \textbf{89.08} & \textbf{100.00} & \textbf{100.00} \\ 
            Random model probing & 51.72 & 51.36 & 52.68 & 35.28 & 55.00 & 57.24 \\
            \midrule
            & \multicolumn{3}{c|}{MATH} & \multicolumn{3}{c}{HaluEval} \\
            & TF & MC & - & TF & MC & - \\ \midrule
            Majority guessing & 50.00 & 50.00 & - & 50.00 & 50.00 & - \\
            \rowcolor{blue!10} \textbf{PT model probing} & \textbf{64.90} & \textbf{60.95} & - & \textbf{60.02} & \textbf{60.81} & - \\ 
            Random model probing & 53.70 & 51.45 & - & 52.31 & 51.13 & - \\
            \bottomrule[1.5px]
            \end{tabular}
    }
    \end{sc}
    \end{small}
  \end{center}
\end{table}

\begin{figure*}[ht]
  \begin{center}
    \centerline{\includegraphics[width=\textwidth]{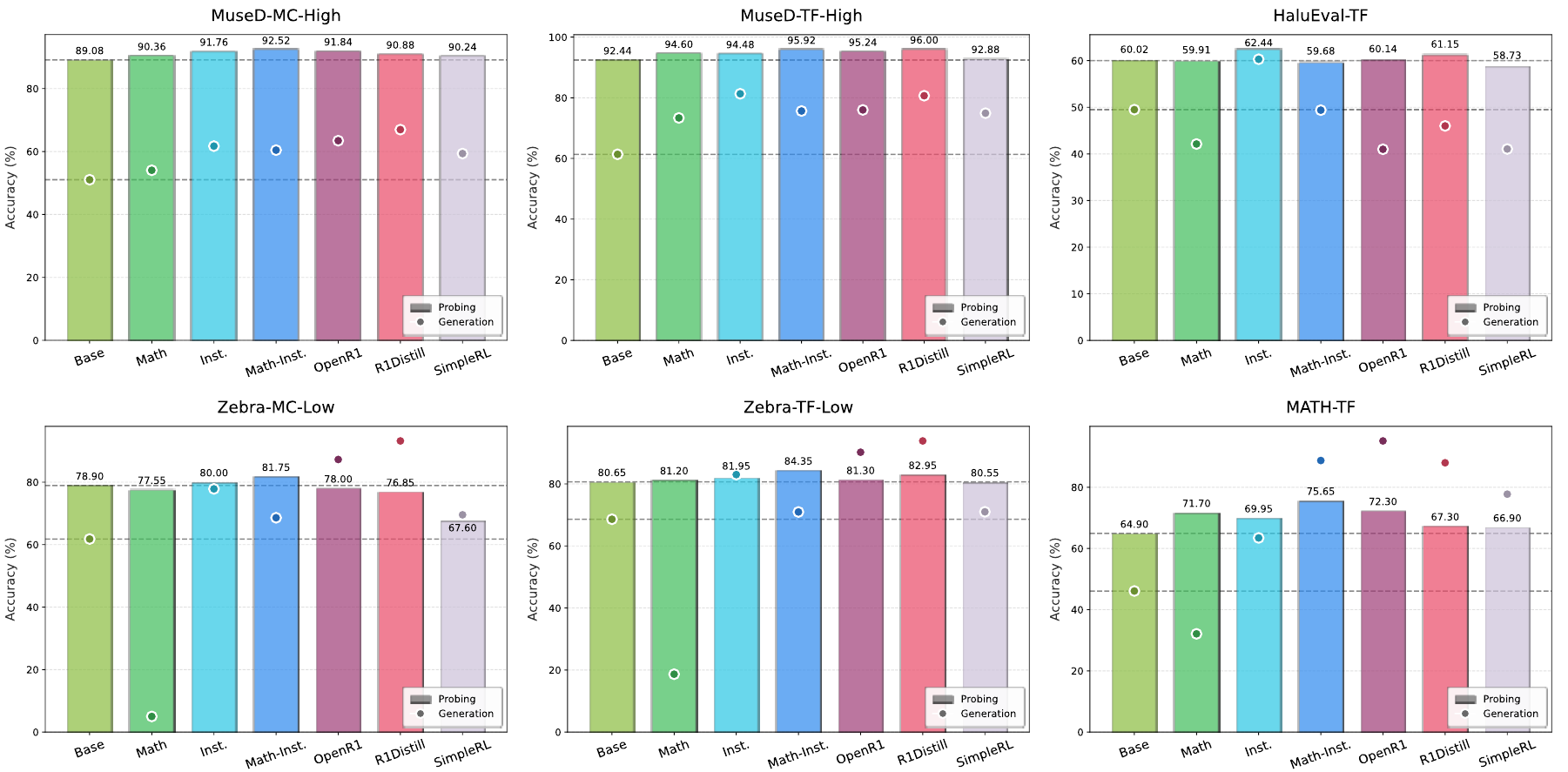}}
    \caption{
    \textbf{Development of initial representation quality and generation accuracy.} The gray dashed lines indicate the baseline performance of the Base model for probing and generation. See Appendix \ref{sec:exp1 more} for full results across all tasks.}
    \label{fig:exp1}
    \vspace{-4ex}
  \end{center}
\end{figure*}

As a preliminary validation, we assess the reliability of the probing methodology by comparing the accuracy of the pre-trained model against a randomly initialized baseline. The results in \cref{tab:baseline} reveal that the pre-trained model significantly outperforms the random model. Crucially, the random baseline's accuracy aligns closely with majority guessing, which always predicts the most frequent label. This confirms that the tasks are sufficiently complex, and the observed high accuracy derives from the model's internal parameters rather than the probe's trainable components. Furthermore, these results yield a notable finding that even in the pre-training phase, prior to any alignment, the model already exhibits significant latent reasoning capacity, enabling it to discern correct solutions before generation.

\begin{table}[ht]
  \caption{\textbf{The maximum performance gain} (Max $\Delta$) represents the highest accuracy improvement (\%) achieved by any other model over the Base model. The ratio indicates the relative scale of probing gains compared to generation gains.}
  \label{tab:delta comparison}
  \begin{center}
    \begin{small}
      \begin{sc}
      \resizebox{\columnwidth}{!}{
            \begin{tabular}{l|ccc|ccc}
            \toprule[1.5px]
             & \multicolumn{3}{c|}{Zebra-TF} & \multicolumn{3}{c}{Zebra-MC} \\
            Max $\Delta$ & High & Med & Low & High & Med & Low \\ \midrule
            Probing gain & 2.05 & 5.30 & 3.70 & 2.95 & 1.60 & 2.85 \\
            Generation gain & 30.05 & 25.75 & 25.25 & 12.75 & 33.20 & 31.40 \\
            \rowcolor{blue!10} \textbf{Ratio (Prob / Gen)} & \textbf{0.07} & \textbf{0.21} & \textbf{0.15} & \textbf{0.23} & \textbf{0.05} & \textbf{0.09} \\ 
            \midrule
            & \multicolumn{3}{c|}{MuseD-TF} & \multicolumn{3}{c}{MuseD-MC} \\
            Max $\Delta$ & High & Med & Low & High & Med & Low \\ \midrule
            Probing gain & 3.56 & 0 & 0 & 3.44 & 0 & 0 \\
            Generation gain & 19.96 & 12.60 & 11.16 & 15.92 & 14.44 & 14.24 \\ 
            \rowcolor{blue!10} \textbf{Ratio (Prob / Gen)} & \textbf{0.18} & \textbf{0.00} & \textbf{0.00} & \textbf{0.22} & \textbf{0.00} & \textbf{0.00} \\ 
            \midrule
            & \multicolumn{3}{c|}{MATH} & \multicolumn{3}{c}{HaluEval} \\
            Max $\Delta$ & TF & MC & - & TF & MC & - \\ \midrule
            Probing gain & 10.75 & 12.05 & - & 2.42 & 4.22 & - \\
            Generation gain & 49.05 & 54.25 & - & 10.81 & 32.15 & - \\
            \rowcolor{blue!10} \textbf{Ratio (Prob / Gen)} & \textbf{0.22} & \textbf{0.22} & - & \textbf{0.22} & \textbf{0.13} & - \\ 
            \bottomrule[1.5px]
            \end{tabular}
    }
    \end{sc}
    \end{small}
  \end{center}
  \vspace{-0ex}
\end{table}

Then we analyze the development of initial representation quality and generation accuracy across tasks, as shown in \cref{fig:exp1}. We observe that post-training yields only marginal gains in representation quality, with improvement typically remaining under $5\%$ across training stages. This contrasts sharply with the substantial gains in generation accuracy. Specifically, as shown in \cref{tab:delta comparison}, the maximum performance gain in probing is less than $25\%$ of the corresponding generation gain. These findings demonstrate that the development of reasoning does not stem primarily from the improvement of the initial internal states formed upon processing a question. Moreover, among reasoning models, those trained via distillation consistently outperform the pure RL variant in both representation quality and generation accuracy, suggesting that distillation may offer a more effective strategy for reasoning enhancement. Notably, initial representation quality frequently surpasses generation accuracy (observed in $10$ out of $16$ tasks, $\approx 60\%$). This indicates that models possess a strong latent reasoning capacity that is often not fully realized in their generation outcomes.

\subsection{Representation Quality Dynamics during Generation}
\label{sec:main exp2}
\begin{figure*}[ht]
  \begin{center}
    \centerline{\includegraphics[width=\textwidth]{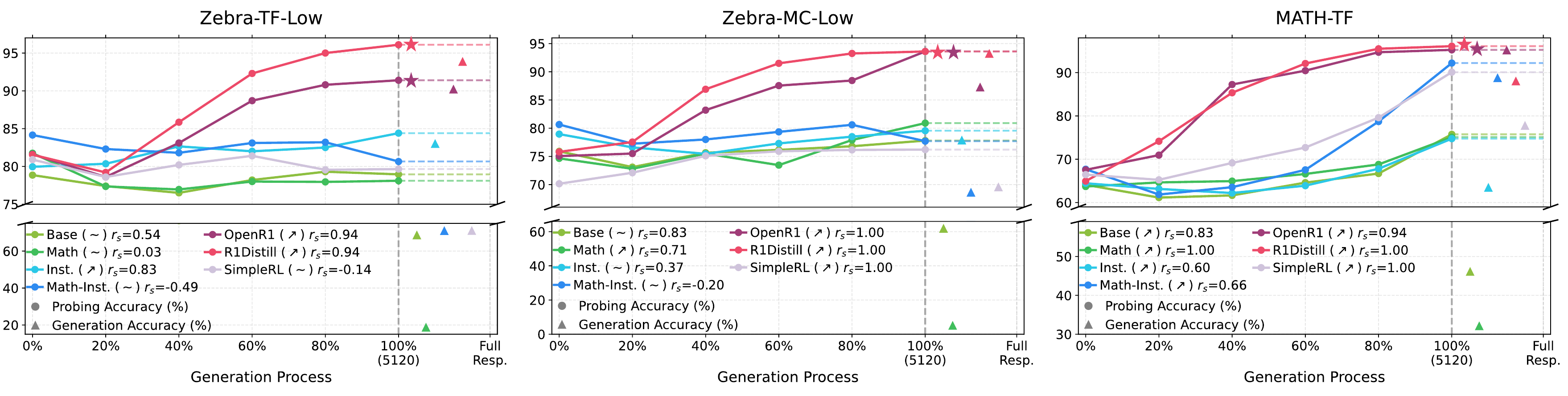}}
    \caption{
      \textbf{Representation quality dynamics during generation.} Trends are analyzed using linear regression and Spearman rank correlation. We mark the final probing accuracy of the strong reasoning models with \ding{72}. See Appendix \ref{sec:exp2 more} for full results.
    }
    \label{fig:exp2}
    \vspace{-5ex}
  \end{center}
\end{figure*}

\begin{figure}[ht]
  \begin{center}
    \centerline{\includegraphics[width=0.67\columnwidth]{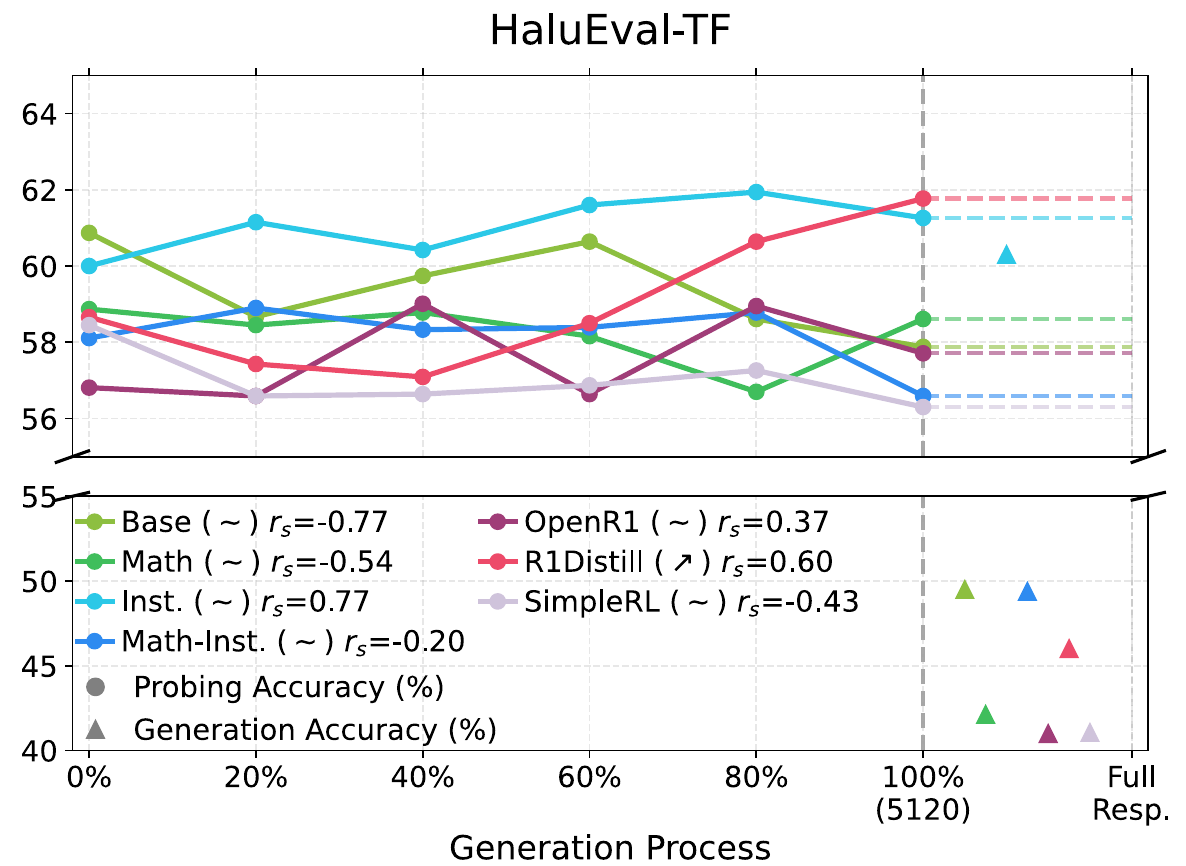}}
    \caption{
      \textbf{Representation quality remains relatively stable} during generation on the factuality task with minimal reasoning.
    }
    \label{fig:exp2halu}
    \vspace{-5ex}
  \end{center}
\end{figure}

Despite the general trend of strong initial representations, we observe that strong reasoning models can achieve generation accuracy that surpasses their initial representation quality, particularly on complex tasks such as Zebra and MATH. This gap motivates us to investigate whether CoT enables the model to progressively improve its internal judgments. To this end, we track probing accuracy across different stages of generation on representative tasks where generation accuracy exceeds the initial representation quality. For comparison, we conduct identical experiments on HaluEval. We split the model's response by ``\textbackslash n'' and construct progressive probing datasets by appending different portions of the CoT. Due to GPU memory constraints, we truncate contexts to the first $5,120$ tokens.

The main results are presented in \cref{fig:exp2} and \cref{fig:exp2halu}, where we quantify probing accuracy trends using linear regression and Spearman rank correlation \cite{spearman1904proof}. For linear regression, we determine the trend ($\nearrow$, $\searrow$) based on the slope if the $p$-value $< 0.1$; otherwise, the trend is classified as fluctuating ($\sim$). Additionally, we employ the Spearman correlation coefficient $r_s$ to assess monotonic relationships without relying on linear assumptions.

\begin{figure}[ht]
  \begin{center}
    \centerline{\includegraphics[width=0.67\columnwidth]{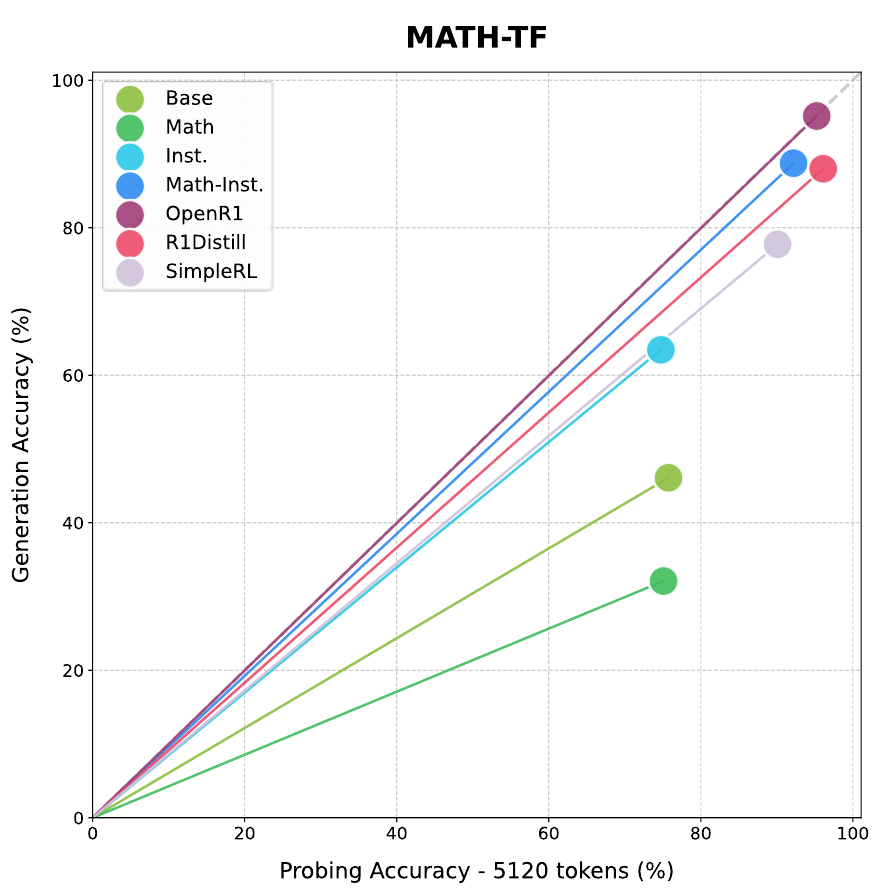}}
    \caption{
      \textbf{Generation accuracy versus probing accuracy.}
    }
    \label{fig:realization}
    \vspace{-6ex}
  \end{center}
\end{figure}

\begin{figure}[ht]
  \begin{center}
    \centerline{\includegraphics[width=0.67\columnwidth]{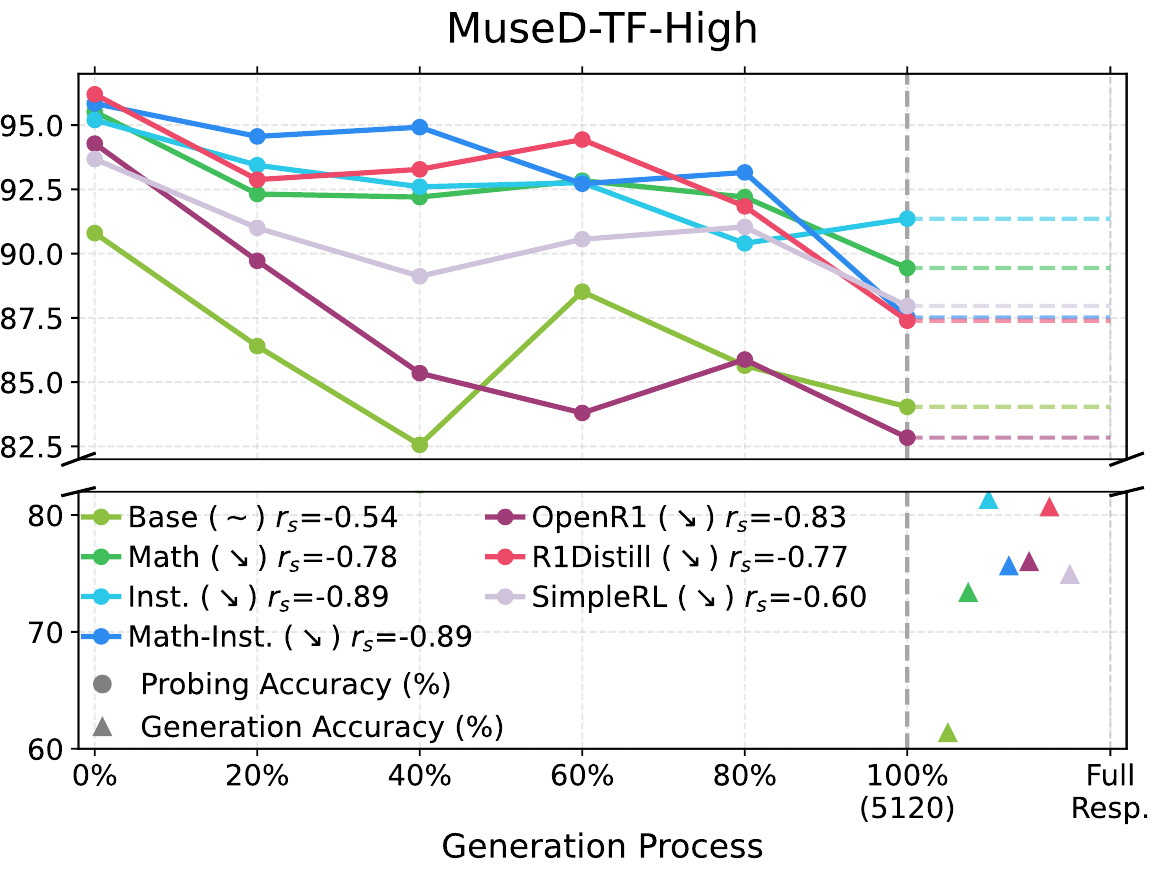}}
    \caption{
      \textbf{Representation quality exhibits a decreasing trend} during generation on the simple reasoning task.
    }
    \label{fig:exp2mused}
    \vspace{-5ex}
  \end{center}
\end{figure}

On the factuality task, which demands minimal reasoning, representation quality remains relatively stable, exhibiting only minor fluctuations throughout the CoT. In contrast, on reasoning tasks, representation quality evolves continuously during generation. We observe that models achieve improved representations following CoT, indicating that the latent reasoning capacity of a single forward pass is often insufficient for complex problems; instead, CoT enables models to iteratively refine internal states. Crucially, advanced reasoning models distinguish themselves by leveraging more effective CoT to attain significantly higher final representation quality, a capability acquired during post-training. Furthermore, as illustrated in \cref{fig:realization}, post-training can also enhance the model's ability to convert these internal representations into generation accuracy. However, a gap persists: while explicit generation accuracy can surpass initial quality, it frequently falls short of the representation quality of the better final state.

We also conduct the experiment on a simple task (MuseD) where the initial representation is already nearly-perfect. As shown in Figure \ref{fig:exp2mused}, we find that the reasoning process can be a double-edged sword, where the representation quality generally exhibits a downward trend, with degradation being particularly significant in strong reasoning models.

\begin{table*}[ht]
  \caption{\textbf{Generation-representation alignment.} Significance levels $P_{MWU}$ of Mann-Whitney U test: NS denotes non-significant ($>5e^{-2}$), * ($<5e^{-2}$), and *** ($<1e^{-10}$). \textbf{Bold}: higher alignment with CoT; \textcolor{blue}{Blue}: R1Distill alignment is lower than Inst. Full results are provided in Appendix \ref{sec:exp3 more}.}
  \label{tab:repr-gen alignment}
  \begin{center}
    \begin{small}
      \begin{sc}
      \resizebox{0.7\textwidth}{!}{
            \begin{tabular}{l|cccc|cccc}
            \toprule[1.5px]
            & \multicolumn{4}{c|}{Inst.} & \multicolumn{4}{c}{R1Distill} \\
            High-Difficulty Task &   Trend    &   $r_s$$\uparrow$   &   ROC-AUC$\uparrow$   &   $P_{MWU}$   &   Trend    &   $r_s$$\uparrow$   &   ROC-AUC$\uparrow$   &   $P_{MWU}$ \\ \midrule
             MuseD-TF        &   $\nearrow$  &   0.59  &  	0.57    &   *    &   $\sim$    &   \textcolor{blue}{0.07}    &  \textcolor{blue}{0.53}     &    ns   \\ \midrule
            \textbf{MuseD-TF+CoT}    &  $\nearrow$   &   \textbf{0.99}   &    \textbf{0.85}  &  ***    &   $\nearrow$    &   \textcolor{blue}{\textbf{0.95}}    &  \textbf{0.88}     &    ***   \\ \midrule
             Zebra-TF & $\nearrow$ & 0.88 & 0.55 & * & $\sim$ & \textcolor{blue}{0.07} & \textcolor{blue}{0.51} & ns \\ \midrule
            \textbf{Zebra-TF+CoT} & $\nearrow$ & \textbf{1.00} & \textbf{0.86} & *** & $\nearrow$ & \textcolor{blue}{\textbf{0.76}} & \textcolor{blue}{\textbf{0.62}} & *** \\ \midrule
            MATH-TF         &   $\nearrow$   &   0.88   &   0.59   &   *   &  $\sim$    &   \textcolor{blue}{0.45}    &    \textcolor{blue}{0.57   }&      * \\ \midrule
            \textbf{MATH-TF+CoT}     &   $\nearrow$   &   \textbf{0.98}   &  \textbf{0.70}   &  ***    &   $\nearrow$    &   \textcolor{blue}{\textbf{0.52}}    &   \textbf{0.90}    &   ***   \\
            \bottomrule[1.5px]
            \end{tabular}
    }
      \end{sc}
    \end{small}
  \end{center}
\end{table*}

\section{Deep Analysis}
While previous experiments establish the dynamic evolution of representation quality during generation, the underlying relationship between explicit token generation and latent representation transitions remains unclear, which we specifically investigate in this section.

\subsection{Statistical Analysis of Alignment}
\label{sec:more exp1}
We first examine the alignment between internal representations and external generation outcomes. We select Inst. and R1Distill as representative non-reasoning and reasoning models, respectively, given their superior performance among post-trained models. To quantify the alignment, we define two variables: the probing probability of the correct label $p \in [0,1]$ and the generation correctness indicator $\delta \in \{0,1\}$. We partition samples into ten equal-width buckets based on $p$, calculate the average accuracy within each bucket, and derive the linear regression trend and Spearman rank correlation coefficient $r_s$. Additionally, we report the sample-level ROC-AUC score and the $p$-value of the Mann-Whitney U test \cite{wilcoxon1945individual, mann1947test} to assess the degree and significance of the divergence in $p$ distributions between correct and incorrect generations.

As presented in \cref{tab:repr-gen alignment}, generation correctness exhibits a weak correlation with the initial representations, as evidenced by low ROC-AUC scores consistently remaining below $0.6$. This demonstrates that the generation process does not merely articulate the model's initial thought, and that the final answer in reasoning scenarios is largely independent of the initial representations. Conversely, generation correlates significantly more strongly with the final representations following CoT. However, the degree of alignment for R1Distill is not higher than that of Inst., indicating that reasoning models do not possess superior efficiency in converting internal representation quality into generation accuracy compared to non-reasoning models.

\begin{table}[ht]
  \caption{\textbf{Alignment between representations before and after CoT}. $r_s$: Spearman rank correlation coefficient; $r_p$: Pearson correlation coefficient; $r^2$ is adopted as the primary metric of linear regression, as $p$-values tend to become uninformative in large datasets. \textcolor{blue}{Blue}: R1Distill alignment is lower than Inst.}
  \label{tab:repr-repr alignment}
  \begin{center}
    \begin{small}
      \begin{sc}
      \resizebox{\columnwidth}{!}{
            \begin{tabular}{l|ccc|ccc}
            \toprule[1.5px]
            & \multicolumn{3}{c|}{Inst.} & \multicolumn{3}{c}{R1Distill} \\
            &   $r_s$$\uparrow$    &   $r_p$$\uparrow$   &   $r^2$$\uparrow$   &   $r_s$$\uparrow$    &   $r_p$$\uparrow$   &   $r^2$$\uparrow$ \\ \midrule
            MuseD-TF-High & 0.16 & 0.13 & 0.02 & \textcolor{blue}{0.10} & \textcolor{blue}{0.05} & \textcolor{blue}{0.00}  \\ \midrule
            Zebra-TF-High & 0.30 & 0.30 & 0.09 & \textcolor{blue}{0.29} & \textcolor{blue}{0.29} & \textcolor{blue}{0.08} \\ \midrule
            MATH-TF & 0.28 & 0.29 & 0.08 & \textcolor{blue}{0.09} & \textcolor{blue}{0.05} & \textcolor{blue}{0.00} \\
            \bottomrule[1.5px]
            \end{tabular}
    }
    \end{sc}
    \end{small}
  \end{center}
  \vspace{-3ex}
\end{table}

Additionally, we investigate the relationship between initial and final representations to quantify the extent of the state transition. The results are shown in \cref{tab:repr-repr alignment}, reporting the Pearson correlation coefficient $r_p$ \cite{pearson1896vii} and linear regression $r^2$, alongside the Spearman rank correlation coefficient $r_s$. Across all metrics, the initial and final representations exhibit extremely low correlation, which confirms that the transition is not simply a minor refinement of the initial states, but rather a fundamental distributional reshaping. Notably, R1Distill consistently exhibits lower correlations than Inst. across all reasoning tasks, indicating that reasoning facilitates a more intense state transition.

\subsection{Counterfactual Analysis of Transition Causes}
\label{sec:more exp2}
\begin{figure*}[ht]
  \begin{center}
    \centerline{\includegraphics[width=\textwidth]{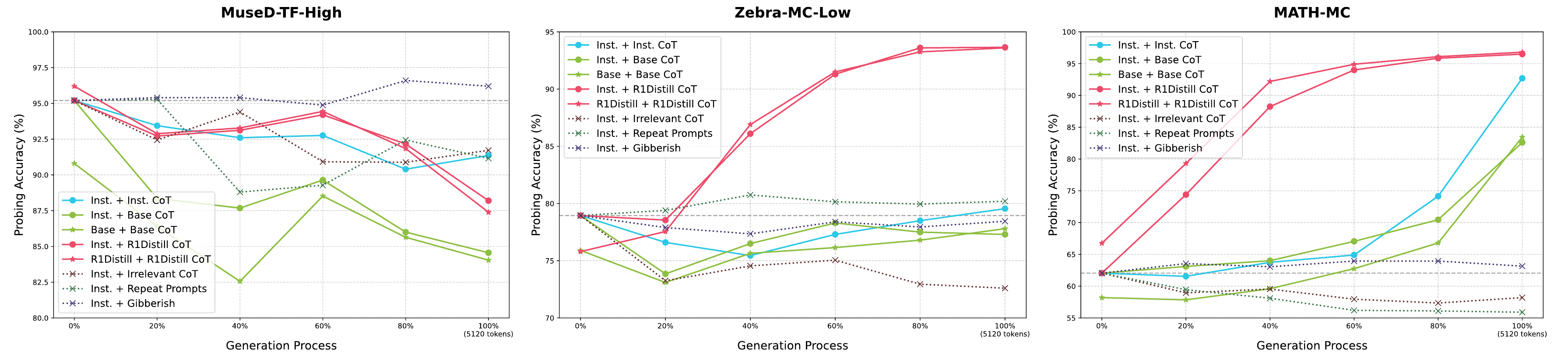}}
    \caption{
      \textbf{Inst. representation quality dynamics across different CoT sources.} The dashed lines indicate the initial probing accuracy.
    }
    \label{fig:comparison}
    \vspace{-4ex}
  \end{center}
\end{figure*}

\begin{figure}[ht]
  \begin{center}
    \centerline{\includegraphics[width=\columnwidth]{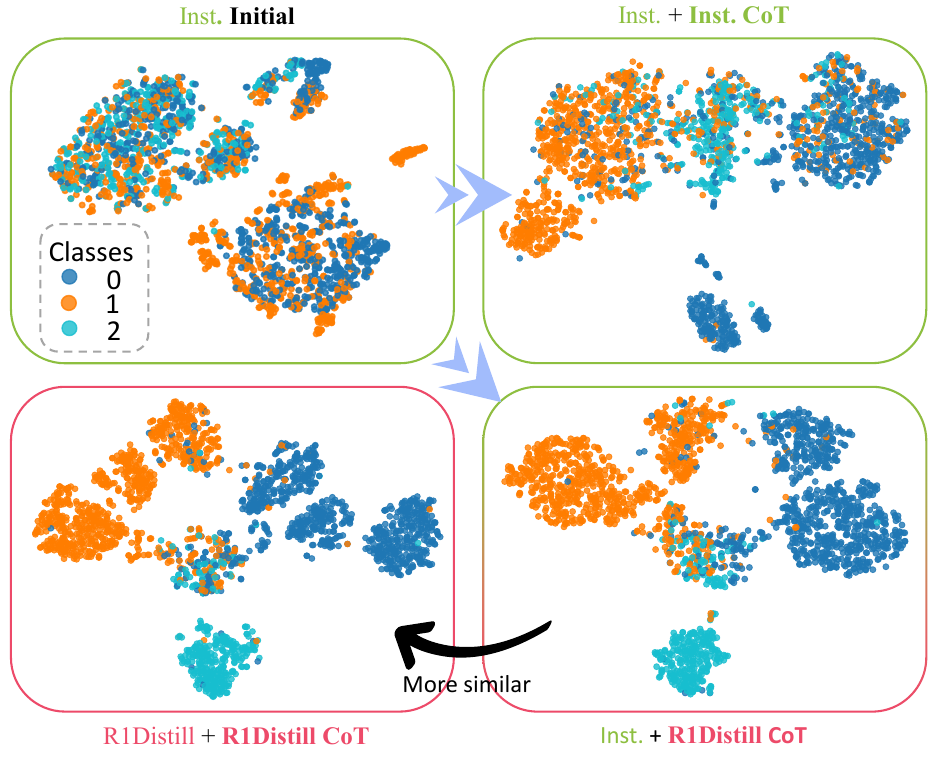}}
    \caption{
      \textbf{Visualization of the distribution shift in Inst. representations} induced by Inst. versus R1Distill CoT on Zebra-MC-Low.
    }
    \label{fig:tsne}
    \vspace{-4ex}
  \end{center}
\end{figure}

Previous sections show that the representations undergo a distributional transition during generation. However, the primary driver of this transition remains unclear. Therefore, we conduct counterfactual experiments to disentangle the effects of the generated content's semantics, the additional computation introduced by CoT inference, and the parameter differences among models at different stages.  

Specifically, we measure the representation quality of the Inst. when conditioned on CoT sequences from different sources: Base, Inst. itself, and R1Distill. To control for computation and context length, we set three baselines: gibberish dot sequences, repeated prompts (which prior work \cite{xu2024re, pfau2024let} suggests may aid reasoning), and CoT on irrelevant problems. For the gibberish and irrelevant baselines, we keep the token length the same as the original Inst. CoT. For the repeated prompt baseline, we repeat the problem statement up to five times. We conduct this analysis on three representative tasks where the Inst. exhibits downward, fluctuating, and rising trends.

\begin{figure}[ht]
  \begin{center}
    \centerline{\includegraphics[width=0.67\columnwidth]{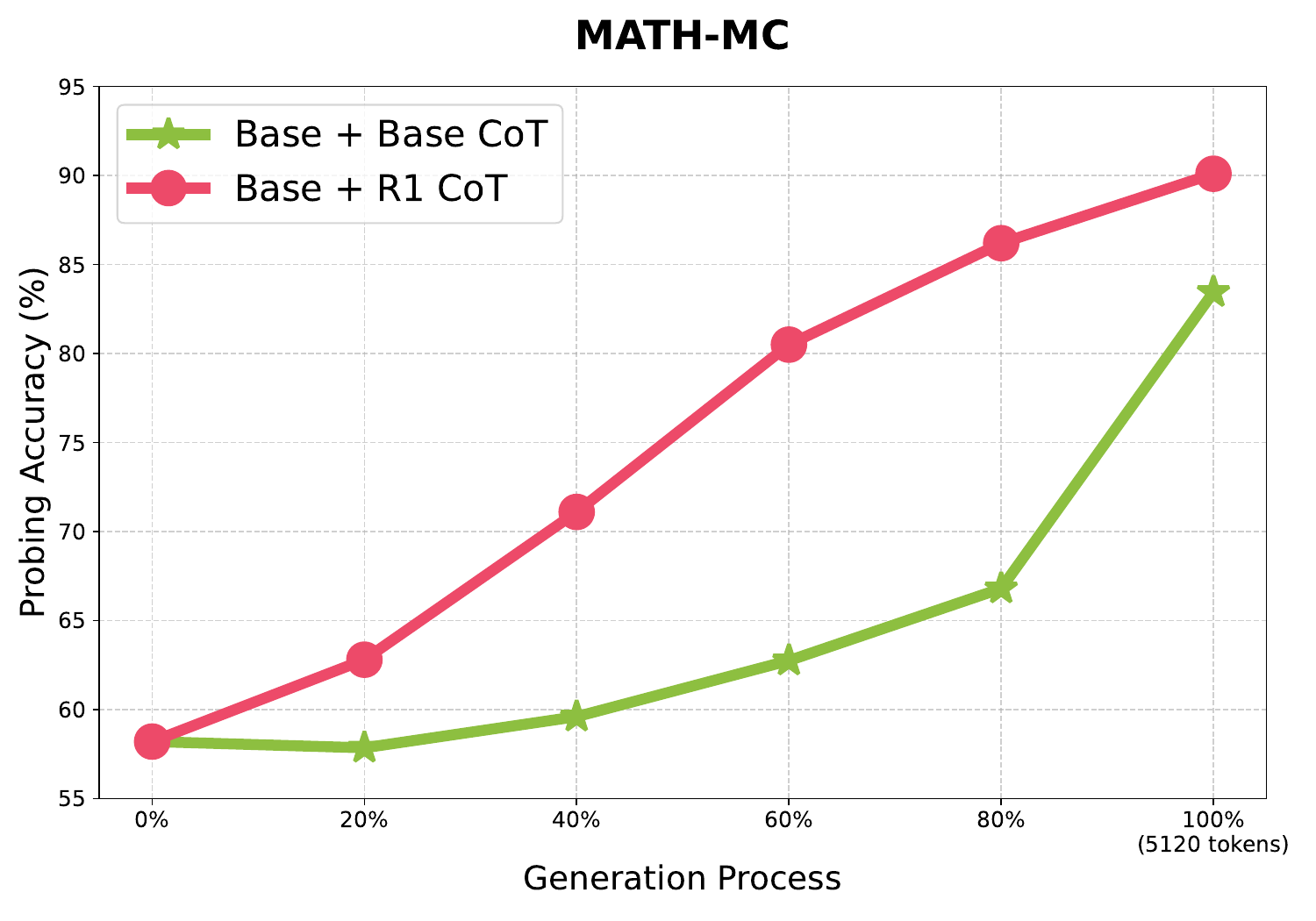}}
    \caption{
      \textbf{Dynamics of Base model representation quality} induced by Base versus R1 CoT.
    }
    \label{fig:comparison-r1}
    \vspace{-3ex}
  \end{center}
\end{figure}

The results in \cref{fig:comparison} demonstrate that the semantics of the generated content is the primary driver of the observed transitions. Merely increasing computation through meaningless and semantically irrelevant content fails to enhance representation quality, and repeated prompts also prove ineffective in the absence of generated responses. Moreover, when Inst. is conditioned on external CoT, its representation quality aligns more closely with that of the source model than with its self-generated CoT. This underscores the dominance of generation semantics over intrinsic parameter differences in shaping representations, suggesting that the model's fundamental representational capacity remains relatively stable throughout post-training. We provide further evidence by visualizing the representation distributions in \cref{fig:tsne} using t-SNE \cite{maaten2008visualizing}, where the distribution of Inst. conditioned on R1Distill CoT overlaps more significantly with that of R1Distill. Additionally, we apply DeepSeek-R1 \cite{guo2025deepseek} CoT to the Base model, and the results in \cref{fig:comparison-r1} reveal that high-quality reasoning paths can elicit superior representations even in pre-trained models. This highlights the primacy of semantics and explains why distillation can effectively empower weak models with strong reasoning capability.

\section{Conclusion and Discussion}
In this paper, we investigate the evolution of model reasoning ability from the perspectives of generation and representation. Through comprehensive experiments, we discover that post-training improves the static initial representation quality only to a limited extent. More importantly, the generation process does not simply verbalize the model's internal thoughts but continuously shifts representations into a new distribution; post-training enables the model to leverage this process more effectively to attain higher final representation quality. Furthermore, we explore the deep relationship between internal representations and external generation through statistical analysis and identify the main driver of the representation transition via counterfactual experiments.
These findings offer valuable insights into reasoning interpretability and highlight promising directions for model enhancement, including: utilizing internal signals for self-training to align generation with representation; designing loss functions at the representation level to enhance model representation quality; controlling model behaviors at test time to prevent representation degradation; and developing optimized generation strategies to attain high-quality representation efficiently.

\section*{Impact Statement}
This paper presents work whose goal is to advance the field of Machine Learning. There are many potential societal consequences of our work, none of which we feel must be specifically highlighted here.

\bibliography{example_paper}
\bibliographystyle{icml2026}

\newpage
\appendix
\onecolumn

\section{Experimental Task Introduction}
\label{sec:app task}
In this section, we detail the settings of the experimental tasks. We describe the construction of both the probing dataset $\mathcal{D}^\text{repr}$ and the generation dataset $\mathcal{D}^\text{gen}$, and provide illustrative examples for each task.

\textbf{ZebraPuzzle} \cite{stojanovski2025reasoninggymreasoningenvironments} is a classic deductive logic puzzle designed to evaluate a model's ability to solve constraint satisfaction problems. Concretely, the problem consists of $n$ houses occupied by $n$ different residents, where each resident is characterized by $m$ attributes. Given a set of constraints, the model must deduce the specific resident associated with each house. 

Since the solution requires the simultaneous resolution of all dependencies, this task effectively serves as a metric for reasoning \textbf{breadth}. Therefore, we stratify the dataset into three difficulty levels based on $n$ and $m$: Low for $n, m \in \{2,3\}$, Med (Medium) for $n, m \in \{3,4\}$, and High for $n, m \in \{4,5\}$. For the TF version, the model must verify whether a specific resident occupies a given house. The dataset is balanced with an equal distribution of correct matches and randomly sampled incorrect matches; For the MC version, the model is required to identify the correct resident of a given house from the list of all possible names. For each difficulty level, we generate a training set of $20,000$ samples and a test set of $2,000$ samples, ensuring a balanced distribution across $(n,m)$ combinations. Illustrative examples are shown in \cref{fig:zebra example}.

\begin{figure*}[htbp]
  \begin{center}
    \centering
    \begin{tcolorbox}[title=Illustrative examples of ZebraPuzzle, width=\textwidth]
        \textbf{Problem description:}\\
        This is a logic puzzle. There are 5 houses (numbered 1 on the left, 5 on the right), from the perspective of someone standing across the street from them. Each has a different person in them. They have different characteristics:\textbackslash n - Each person has a unique name: arnold, alice, david, carol, bob\textbackslash n - Everyone has a different favorite cigar: red eye, dunhill, blue master, pall mall, prince\textbackslash n - They all have a different favorite flower: lilies, daffodils, iris, carnations, orchid\textbackslash n - Everyone has something different for lunch: pizza, spaghetti, stir fry, grilled cheese, soup\textbackslash n - The people keep different animals: cat, dog, horse, fish, bird\textbackslash n\textbackslash n1. The cat lover is the person who loves a bouquet of daffodils.\textbackslash n2. The person who loves the soup is the person who smokes Blue Master.\textbackslash n3. Carol and the person who smokes Red Eye are next to each other.\textbackslash n4. The person who loves the boquet of orchid is directly left of the person who loves a bouquet of daffodils.\textbackslash n5. The person who loves stir fry is in the second house.\textbackslash n6. The person who loves the boquet of lilies is Alice.\textbackslash n7. The Dunhill smoker is the person who keeps horses.\textbackslash n8. Bob is directly left of Carol.\textbackslash n9. Alice is the Dunhill smoker.\textbackslash n10. The dog owner is the person who smokes Red Eye.\textbackslash n11. The bird keeper is in the fourth house.\textbackslash n12. The person who loves the boquet of iris is the person who loves the spaghetti eater.\textbackslash n13. Arnold is in the second house.\textbackslash n14. The person who loves eating grilled cheese is directly left of the person who smokes Blue Master.\textbackslash n15. The person who loves a bouquet of daffodils is the Prince smoker.\\ \\
        \textbf{TF version:}\\
        \textbf{Probe:} <|START|>House 1<|MID|>alice<|END|> \textit{(0 for False and 1 for True.)} \\
        \textbf{Question:} Decide whether alice is the name of the person who lives in House 1.\\
        \textbf{Correct answer:} 1 / True  \textit{(alice live in House 1)}\\\\
       \textbf{MC version:}\\
        \textbf{Probe:} <|START|>House 1<|END|> \textit{(0, 1, 2, 3, 4 for alice, arnold, bob, carol, david, respectively.)} \\
        \textbf{Question:} What is the name of the person who lives in House 1? \\
        \textbf{Correct answer:} 0 / alice  \textit{(alice live in House 1)}
    \end{tcolorbox}
    \caption{
      Illustrative examples of \textbf{ZebraPuzzle}.
    }
    \label{fig:zebra example}
  \end{center}
\end{figure*}

\textbf{MuseD} \cite{li2024boostingdeductivereasoningstep} is a deductive reasoning task designed to evaluate a model's ability to solve multi-hop syllogistic reasoning. Concretely, the problem presents a set of \textit{premises} defining relationships between various entities, utilizing the four classical Aristotelian categorical propositions, where $S$ and $P$ denote the \textit{Subject} and \textit{Predicate}:
\begin{itemize}
    \item Type 1 (\textit{A}): All S are P.
    \item Type 2 (\textit{E}): All S are not P.
    \item Type 3 (\textit{I}): There exists one S that is P.
    \item Type 4 (\textit{O}): There exists one S that is not P.
\end{itemize}
Solving the problem requires chaining valid syllogisms (e.g., All \textit{M} are \textit{P} $+$ All \textit{S} are \textit{M}  $\implies$ All $S$ are $P$) to derive the final relationship between a target subject and predicate.

Since the solution requires a sequence of $d$ deductive steps, this task effectively serves as a metric for reasoning \textbf{depth}. Therefore, we stratify the dataset into three difficulty levels based on the information structure: Low, where premises are all necessary and presented in the optimal deductive order; Med, where premises are all necessary but permuted randomly; High, where premises are permuted randomly and include 20\% distraction noise (irrelevant relationships). For the TF version, the model must verify the correctness of a specific relationship type between two entities. The dataset is balanced with an equal distribution of correct relationships and incorrect matches (Type A $\leftrightarrow$ Type O, Type E $\leftrightarrow$ Type I). For the MC version, the model is required to identify the best relationship type for a given pair of entities. For each difficulty level, we generate a training set of $20,000$ samples and a test set of $2,500$ samples, covering deduction depths $d \in [3,25]$ and ensuring a balanced distribution across four relationship types. Illustrative examples are shown in \cref{fig:mused example}.

\begin{figure*}[htbp]
  \begin{center}
    \centering
    \begin{tcolorbox}[title=Illustrative examples of MuseD, width=\textwidth]
        \textbf{Problem description:}\\
        Given:\textbackslash nAll Q are not N.\textbackslash nThere exists one Gamma that is W.\textbackslash nAll S are Q.\textbackslash nAll W are Rho.\textbackslash nAll Rho are S.\\ \\
        \textbf{TF version:}\\
        \textbf{Probe:} <|START|>W<|MID1|>N<|MID2|>2<|END|> \textit{(0 for False and 1 for True.)} \\
        \textbf{Question:} Decide whether the statement is true or false: All W are not N.\\
        \textbf{Correct answer:} 1 / True  \textit{(The relationship type is E)}\\\\
       \textbf{MC version:}\\
        \textbf{Probe:} <|START|>W<|MID|>N<|END|> \textit{(0, 1, 2, 3 for A, E, I, O, respectively.)} \\
        \textbf{Question:} Decide the relationship between W and N.\textbackslash n\textbackslash nType 1: All W are N.\textbackslash nType 2: All W are not N.\textbackslash nType 3: There exists one W that is N.\textbackslash nType 4: There exists one W that is not N.\textbackslash n\textbackslash nChoose the single best option that describes the relationship between W and N. \\
        \textbf{Correct answer:} 1 / Type 2  \textit{(The relationship type is E)}
    \end{tcolorbox}
    \caption{
      Illustrative examples of \textbf{MuseD}.
    }
    \label{fig:mused example}
  \end{center}
\end{figure*}

\textbf{MATH} \cite{hendrycksmath2021} evaluates a model's reasoning ability on \textbf{unstructured} and \textbf{real-world} mathematical problems. Unlike the limited output spaces of the previous two tasks, mathematical questions typically feature open-ended numerical or algebraic solutions, which are incompatible with fixed-class probing. To adapt the dataset for our framework, we implement a rule-based algorithm to generate plausible incorrect distractors, converting the task into answer verification. 

For the TF version, the model must verify whether a provided candidate answer is correct. The dataset is balanced with an equal distribution of ground truth answers and generated incorrect distractors. For the MC version, the model is required to identify the correct answer from a pair of candidates. We construct the training set using the $7,500$ samples from the original MATH training split and create a test set of $2,000$ samples randomly drawn from the original test set. Illustrative examples are shown in \cref{fig:math example}.

\begin{figure*}[htbp]
  \begin{center}
    \centering
    \begin{tcolorbox}[title=Illustrative examples of MATH, width=\textwidth]
        \textbf{Problem description:}\\
        Given $x = -2$ find the value of $2x^2+3x+4$.\\ \\
        \textbf{TF version:}\\
        \textbf{Probe:} <|ANS|>7<|END|> \textit{(0 for False and 1 for True.)} \\
        \textbf{Question:} I want you act as an answer judge. Given a math question and a candidate answer, your objective is to determine if the provided answer is correct or not.\textbackslash nQuestion: Given $x = -2$ find the value of $2x^2+3x+4$.\textbackslash nCandidate answer: 7\\
        \textbf{Correct answer:} 0 / False  \textit{(The correct answer is 6)}\\\\
       \textbf{MC version:}\\
        \textbf{Probe:} <|CA1|>7<|CA2|>6<|END|> \textit{(0, 1 for the first and the second candidate answers, respectively.)} \\
        \textbf{Question:} A. 7\textbackslash nB.6\textbackslash n\textbackslash nWhich answer is correct? \\
        \textbf{Correct answer:} 1 / B  \textit{(The correct answer is 6)}
    \end{tcolorbox}
    \caption{
      Illustrative examples of \textbf{MATH}.
    }
    \label{fig:math example}
  \end{center}
\end{figure*}

\textbf{HaluEval} \cite{li2023halueval} is a factuality benchmark to evaluate \textit{hallucination} detection in general knowledge queries. It serves as a control task that requires little reasoning ability, relying primarily on knowledge retrieval. We initially randomly split the dataset into a training set of $8,000$ samples and a test set of $2,000$ samples. However, we observe a significant format bias in the original dataset: ground truth answers are typically concise entities, whereas hallucinated answers are complete, lengthy sentences. To prevent models from exploiting answer length as a spurious shortcut, we employ GPT-4o \cite{hurst2024gpt} to rewrite the hallucinated answers into the same concise format as the ground truth answers. After filtering out instances where the rewritten hallucinated answer becomes semantically the same as the correct answer, $7,037$ training samples and $1,776$ test samples are retained. 

Similar to MATH, for the TF version, the model must verify whether a provided candidate answer is factually correct. The dataset is balanced with an equal distribution of correct answers and hallucinated answers. For the MC version, the model is required to identify the correct answer from a pair of candidates. Illustrative examples are shown in \cref{fig:halu example}.

\begin{figure*}[htbp]
  \begin{center}
    \centering
    \begin{tcolorbox}[title=Illustrative examples of HaluEval, width=\textwidth]
        \textbf{Problem description:}\\
        Dua Lipa, an English singer, songwriter and model, the album spawned the number-one single \"New Rules\" is a song by English singer Dua Lipa from her eponymous debut studio album, released in what year?\\ \\
        \textbf{TF version:}\\
        \textbf{Probe:} <|ANS|>2018<|END|> \textit{(0 for False and 1 for True.)} \\
        \textbf{Question:} \textit{HaluEval template} $+$ Question: Dua Lipa, an English singer, songwriter and model, the album spawned the number-one single \"New Rules\" is a song by English singer Dua Lipa from her eponymous debut studio album, released in what year?\textbackslash nAnswer: 2018\textbackslash nYour Judgement: \\
        \textbf{Correct answer:} 0 / False  \textit{(The correct answer is 2017)}\\\\
       \textbf{MC version:}\\
        \textbf{Probe:} <|CA1|>2017<|CA2|>2018<|END|> \textit{(0, 1 for the first and the second candidate answers, respectively.)} \\
        \textbf{Question:} A. 2017\textbackslash nB.2018\textbackslash n\textbackslash nWhich answer is correct? \\
        \textbf{Correct answer:} 0 / A  \textit{(The correct answer is 2017)}
    \end{tcolorbox}
    \caption{
      Illustrative examples of \textbf{HaluEval}.
    }
    \label{fig:halu example}
  \end{center}
\end{figure*}

\section{Model Introduction}
\label{sec:app model}
In this section, we detail the training relationships within our experimental models, including Qwen2.5-7B, -1.5B series \cite{qwen2.5, yang2024qwen25mathtechnicalreportmathematical} and Llama3.1-8B series \cite{dubey2024llama}. We primarily conduct experiments and analysis on the Qwen2.5-7B series, as its diverse variants spanning multiple training stages provide an ideal testbed for comprehensively analyzing the evolution of reasoning ability.

\textbf{Qwen2.5-7B.} In this work, we evaluate a suite of seven different models from the Qwen2.5-7B series, spanning three distinct developmental stages: pre-training, standard instruct tuning, and reasoning-focused training.
\begin{itemize}
    \item Pre-trained models: Qwen2.5-7B-Base, the foundational general purpose model. Building on this, Qwen2.5-7B-Math undergoes continuous pre-training on a specialized corpus of mathematical corpus and synthetic data from Qwen2-Math-Instruct \cite{yang2024qwen2}. Pre-trained models represent the base capability of the model series prior to any alignment.
    \item Instruction-tuned models: Derived from the Base and Math models, Qwen2.5-7B-Instruct and Qwen2.5-7B-Math-Instruct, respectively, undergo a standard post-training pipeline that includes Supervised Fine-Tuning (SFT) and RL to enhance instruction following ability. While we also refer to these as ``non-reasoning'' models in the main body to distinguish them from the reasoning models, they naturally retain reasoning ability to some extent.
    \item Reasoning-optimized models: We include three specialized models designed to maximize reasoning performance. Qwen2.5-7B-OpenR1 \cite{openr1} and Qwen2.5-7B-R1Distill \cite{guo2025deepseek} are distilled models based on Math-Instruct and Math models, respectively. They are fine-tuned (SFT) on reasoning responses generated by DeepSeek-R1 \cite{guo2025deepseek} without additional RL. Conversely, Qwen2.5-7B-SimpleRL-Zoo \cite{zeng2025simplerl} applies RL (specifically GRPO \cite{shao2024deepseekmath}) directly to the Math model. These models typically exhibit superior performance on tasks necessitating multi-hop reasoning and analysis compared to their predecessors.
\end{itemize}

\textbf{Qwen2.5-1.5B.} The Qwen2.5-1.5B series follows a training pipeline identical to that of the 7B series, with the exception of the OpenR1 variant, which does not exist at this scale. Therefore, we include the remaining six corresponding models in our experimental analysis. 

\textbf{Llama3.1-8B.} We evaluate four different variants of the Llama3.1-8B series. In addition to the Base model, we include the Instruct, R1Distill, and SimpleRL variants, which are derived from the Base model via instruction tuning, distillation, and RL, respectively.

\section{Probing Settings}
\label{sec:app probing}
In this section, we detail the implementation of the V-probing technique \cite{ye2024physics} and specify the experimental hyperparameters and settings to ensure reproducibility. Furthermore, we provide an information-theoretic discussion suggesting that, for tasks of sufficient complexity, the additional trainable parameters in V-probing alone lack the capacity to solve the task.

\subsection{Training Settings}
\label{sec:app probing training settings}
We implement V-probing using Low-Rank Adaption (LoRA) \cite{lora}. Specifically, we adopt the configuration with a small rank of $r=4$, a scaling factor of $\alpha=16$, a dropout ratio of $p=0.1$, and no bias terms. Optimization is performed using the standard Adam optimizer \cite{Adam} with a learning rate of $\gamma=1e^{-4}$ and a batch size of $b=32$ by default. All models are initialized in bfloat16 precision.

For the evaluation of initial representation quality in section \ref{sec:main exp1}, we train the model for a fixed budget of $N=10$ epochs to ensure sufficient training ($N=30$ epochs for MATH and HaluEval tasks that have less training data). The input $x$ is constructed by concatenating the probe trigger after the problem description, and the label is predicted using the hidden state of the last token at the final layer. When evaluating the development of representation quality during the CoT process in section \ref{sec:main exp2}, we dynamically adjust the training epochs based on the convergence speed observed in the initial experiments. In this experimental setting, we introduce another special token \texttt{<|Reasoning|>} to separate the problem statement and the CoT, where the input $x$ is formatted as: \texttt{problem description} $+$ \texttt{<|Reasoning|>} $+$ \texttt{CoT} $+$ \texttt{probe trigger}. 

For the comparative analysis in section \ref{sec:more exp2}, we establish three baseline conditions to evaluate the necessity of meaningful CoT content: gibberish sequences, repeated prompts, and irrelevant CoT. In each experimental setting, we replace the original \texttt{CoT} with a corresponding baseline sequence. For gibberish sequences, we substitute the \texttt{CoT} part with a sequence of repeated dots (`` .'') of the same length; For repeated prompts, we replace the \texttt{CoT} part with the probe trigger appended with the problem description, repeated $1$ to $5$ times; For irrelevant CoT, we substitute the \texttt{CoT} part with a CoT generated for an unrelated decryption task, truncated to match the original token length.

\subsection{Probing Algorithm}
\label{sec:app probing algorithm}
The pseudocode for the V-probing implementation is presented in \cref{alg:probing}.

\begin{algorithm}[htbp]
  \caption{Representation Quality Quantification via V-probing}
  \label{alg:probing}
  \begin{algorithmic}[1]
    \STATE {\bfseries Input:} Dataset $\mathcal{D}^\text{repr}$ split into $\mathcal{D}^\text{repr}_\text{train}$ and $\mathcal{D}^\text{repr}_\text{test}$, number of special tokens $k$, number of classes $s$; model component $f_\theta$ with hidden size $m$; learning rate $\gamma$, LoRA rank $r$, number of epochs $N$.
    \STATE
    \STATE {\bfseries Training:}
    \STATE Freeze backbone parameters $\theta$
    \STATE Initialize trainable parameters $\Phi = \{E_\text{sp}, \Delta \theta_{\text{emb}}, h\}$, where:
        \STATE \quad $E_\text{sp} \in \mathbb{R}^{k \times m}$ (Special token embeddings)
        \STATE \quad $\Delta \theta_{\text{emb}}$ (LoRA adapter for embeddings, rank $r$)
        \STATE \quad $h \in \mathbb{R}^{m \times s}$ (Linear probe head)
    \FOR{epoch $= 1$ {\bfseries to} $N$}
        \FOR{each batch $\mathcal{B} = \{(x^{(i)}, z^{(i)})\} \subset \mathcal{D}^\text{repr}_\text{train}$}
            \STATE Compute representations: $c^{(i)} = f_{\theta, E_\text{sp}, \Delta \theta_{\text{emb}}}(x^{(i)})$
            \STATE Estimate distributions: $p_\Phi(z | x^{(i)}) = \text{softmax}(h(c^{(i)}))$
            \STATE Compute Loss: $\mathcal{L}_{\text{CE}} = \frac{1}{|\mathcal{B}|} \sum_{(x^{(i)}, z^{(i)}) \in \mathcal{B}} -\log p_\Phi(z^{(i)} | x^{(i)})$
            \STATE Update $\Phi \leftarrow \Phi - \gamma \nabla_\Phi \mathcal{L}_{\text{CE}}$
        \ENDFOR
    \ENDFOR
    \STATE
    \STATE {\bfseries Evaluation:}
    \STATE Initialize correct count $C \leftarrow 0$
    \FOR{each batch $\mathcal{B} = \{(x^{(i)}, z^{(i)})\} \subset \mathcal{D}^\text{repr}_\text{test}$}
        \STATE Compute representations: $c^{(i)} = f_{\theta, E_\text{sp}, \Delta \theta_{\text{emb}}}(x^{(i)})$
        \STATE Estimate distributions: $p_\Phi(z | x^{(i)}) = \text{softmax}(h(c^{(i)}))$
        \STATE Predict label: $\hat{z}^{(i)} = \arg\max_z p_\Phi(z | x^{(i)})$
        \STATE $C \leftarrow C + \sum_{(x^{(i)}, z^{(i)}) \in \mathcal{B}} \mathbf{1}(\hat{z}^{(i)} = z^{(i)})$
    \ENDFOR
    \STATE \textbf{Return} $Acc_{\text{repr}} = C / |\mathcal{D}^\text{repr}_\text{test}|$
  \end{algorithmic}
\end{algorithm}

\subsection{Information-Theoretic Analysis of V-probing}
\label{sec:probing_proof}
In this section, we analyze the capacity of the V-probing parameters to fit the task data. We establish an upper bound on the achievable training accuracy based on information theory, under the assumption that the model's pre-existing parameters possess no knowledge of the target labels. We demonstrate that, for tasks of sufficient complexity, the probe's limited parameter budget lacks the capacity to reach high accuracy independently. The following derivation focuses on the simplest binary classification setting.

\textbf{Assumption 1 (Random Labels).} Let the target labels $Y^N = (Y_1, \dots, Y_N) \in \{0, 1\}^N$ be independent and identically distributed (\textit{i.i.d.}) random variables following a Bernoulli distribution with $p=0.5$, where $N$ denotes the dataset size. Consequently, the total label entropy is $H(Y^N) = N$ bits. This assumption represents the scenario where the task is sufficiently complex, containing $N$ bits of incompressible entropy, and the model's existing parameters provide no informative signal. 

\textbf{Assumption 2 (Limited Capacity).} Let $\Theta$ denote the trainable parameters of the probe. We bound the effective information capacity of these parameters by $P_{eff}$ bits, such that $H(\theta) \le P_{eff}$.

\begin{proposition}[Capacity-Constrained Accuracy Bound]
\label{prop:bound}
Let $Acc$ be the random variable representing the training accuracy of the probe with parameter capacity $P_{eff}$ bits on a balanced binary classification dataset of size $N$. Under the two assumptions, the expected accuracy is upper bounded by:
\begin{equation}
\mathbb{E}[Acc] \le \frac{1}{2} + \sqrt{\frac{\ln 2}{2} \cdot \frac{P_{eff}}{N}}
\end{equation}
\end{proposition}

\textit{Proof}
\begin{enumerate}
    \item \textbf{Global Information Constraint:} 
    Let $\hat{Y}^N$ denote the sequence of predictions. The mutual information between the true labels and the predictions is bounded by the capacity of the probe parameters $\Theta$:
    \begin{equation}
    \label{eq:1}
    I(Y^N; \hat{Y}^N) \le I(Y^N; \Theta) \le H(\Theta) \le P_{eff} 
    \end{equation}
    
    \item \textbf{Conditional Entropy Lower Bound:} 
    Since $I(Y^N; \hat{Y}^N) = H(Y^N) - H(Y^N | \hat{Y}^N)$ and $H(Y^N) = N$, we have:
    \begin{equation}
    \label{eq:2}
    H(Y^N | \hat{Y}^N) = N - I(Y^N; \hat{Y}^N) \ge N - P_{eff}
    \end{equation}
    
    \item \textbf{Decomposition and Fano's Inequality:} 
    The joint conditional entropy is upper bounded by the sum of marginal conditional entropies:
    \begin{equation}
    \label{eq:3}
    H(Y^N | \hat{Y}^N) \le \sum_{i=1}^N H(Y_i | \hat{Y}_i)
    \end{equation}
    Let $P_{e,i} = \mathbb{P}(\hat{Y}_i \neq Y_i)$ be the marginal error probability for the $i$-th sample. By Fano's inequality for binary variables, $H(Y_i | \hat{Y}_i) \le h_2(P_{e,i})$, where $h_2(\cdot)$ is the binary entropy function. Thus:
    \begin{equation}
    \label{eq:4}
    H(Y^N | \hat{Y}^N) \le \sum_{i=1}^N h_2(P_{e,i})
    \end{equation}
    
    \item \textbf{Applying Jensen's Inequality:} 
    Let the expected error rate be $\bar{P}_e = \frac{1}{N} \sum_{i=1}^N P_{e,i} = 1 - \mathbb{E}[Acc]$.
    Since the binary entropy function $h_2(p)$ is concave, Jensen's Inequality implies that the average of the function values is less than or equal to the function of the average:
    \begin{equation}
    \label{eq:5}
    \frac{1}{N} \sum_{i=1}^N h_2(P_{e,i}) \le h_2\left(\frac{1}{N} \sum_{i=1}^N P_{e,i}\right) = h_2(\bar{P}_e)
    \end{equation}
    Combining \cref{eq:2}, \cref{eq:4}, and \cref{eq:5}:
    \begin{equation}
    \label{eq:6}
    N - P_{eff} \le N \cdot h_2(\bar{P}_e) \implies h_2(1 - \mathbb{E}[Acc]) \ge 1 - \frac{P_{eff}}{N}
    \end{equation}
    Note that $h_2(1 - x) = h_2(x)$, so we have $h_2(\mathbb{E}[Acc]) \ge 1 - \frac{P_{eff}}{N}$.
    
    \item \textbf{Pinsker's Inequality Bound:} 
    We relate the binary entropy to the classification accuracy deviation using Pinsker's inequality. The KL divergence between a Bernoulli variable with parameter $p$ and the uniform distribution is $1 - h_2(p)$. Pinsker's inequality states:
    \begin{equation}
    \label{eq:7}
    1 - h_2(p) \ge \frac{2}{\ln 2} (p - 0.5)^2
    \end{equation}
    Substituting $p = \mathbb{E}[Acc]$ and rearranging:
    \begin{equation}
    \label{eq:8}
    h_2(\mathbb{E}[Acc]) \le 1 - \frac{2}{\ln 2} (\mathbb{E}[Acc] - 0.5)^2
    \end{equation}
    
    \item \textbf{Final Derivation:} 
    Substituting \cref{eq:8} into \cref{eq:6}:
    \begin{equation}
    \label{eq:9}
    1 - \frac{P_{eff}}{N} \le 1 - \frac{2}{\ln 2} (\mathbb{E}[Acc] - 0.5)^2
    \end{equation}
    \begin{equation}
    \label{eq:10}
    \frac{2}{\ln 2} (\mathbb{E}[Acc] - 0.5)^2 \le \frac{P_{eff}}{N}
    \end{equation}
    Solving for $\mathbb{E}[Acc]$ gives the result:
    \begin{equation}
    \label{eq:11}
    \mathbb{E}[Acc] \le 0.5 + \sqrt{\frac{\ln 2}{2} \cdot \frac{P_{eff}}{N}}
    \end{equation}

    So far, we have proven the \cref{prop:bound}.\hfill $\square$
\end{enumerate}

The term $\sqrt{\frac{P_{eff}}{N}}$ decays as the dataset size $N$ increases relative to the probe's capacity $P_{eff}$. Therefore, if the model's existing parameters do not contain any knowledge about the complex task and the sample size is sufficiently large, then the V-probing accuracy will converge to random guessing. Conversely, observing a high training accuracy implies that the hypothesis is false, i.e., the model's existing parameters must contain significant information about the target labels.

\section{Generation Settings}
\label{sec:app gen}
In this section, we detail the generation settings used to evaluate generation accuracy.

We employ vLLM \cite{kwon2023efficient} for generation, loading all models in bfloat16 precision to align with the probing settings. A context window of $32,768$ tokens is allocated to ensure sufficient space for long reasoning chains. For decoding configurations, we apply the default temperature for pre-trained models and instruction-tuned (non-reasoning) models. For reasoning models, we adopt the widely utilized configuration of temperature $0.6$ and top-p $0.95$ following \citet{guo2025deepseek, openr1}. To construct the input, we append the specific question after the problem description. To facilitate automated answer parsing, we prompt the model to enclose its final answer in ``\textbackslash\textbackslash boxed\{\}''. While most models utilize their default chat templates, for models lacking a native chat template (Llama3.1-8B-Base and Llama3.1-8B-SimpleRL) or those with limited instruction-following ability (Qwen2.5-1.5B pre-trained and SimpleRL variants), we adopt the simple chat template proposed by \citet{zeng2025simplerl}.

\section{Additional Experimental Results}
\label{sec:app result}
In this section, we present supplementary experimental results to support the main analysis. We provide the complete results for Qwen2.5-7B series, as well as the results from the replication experiments conducted on the Qwen2.5-1.5B and Llama3.1-8B series.

\subsection{Development of Initial Representation Quality}
\label{sec:exp1 more}
In this section, we present the complete experimental results of section \ref{sec:main exp1} (\cref{fig:exp1-qwen-7-all-1} and \cref{fig:exp1-qwen-7-all-2}). Additionally, we detail the replication experiments conducted on representative tasks for Llama3.1-8B series (\cref{fig:exp1-llama-8-all}, \cref{tab:llama delta comparison}) and Qwen2.5-1.5B series (\cref{fig:exp1-qwen-15-all}, \cref{tab:qwen 1.5 delta comparison}).

\begin{figure*}[htbp]
  \begin{center}
    \centerline{\includegraphics[width=\textwidth]{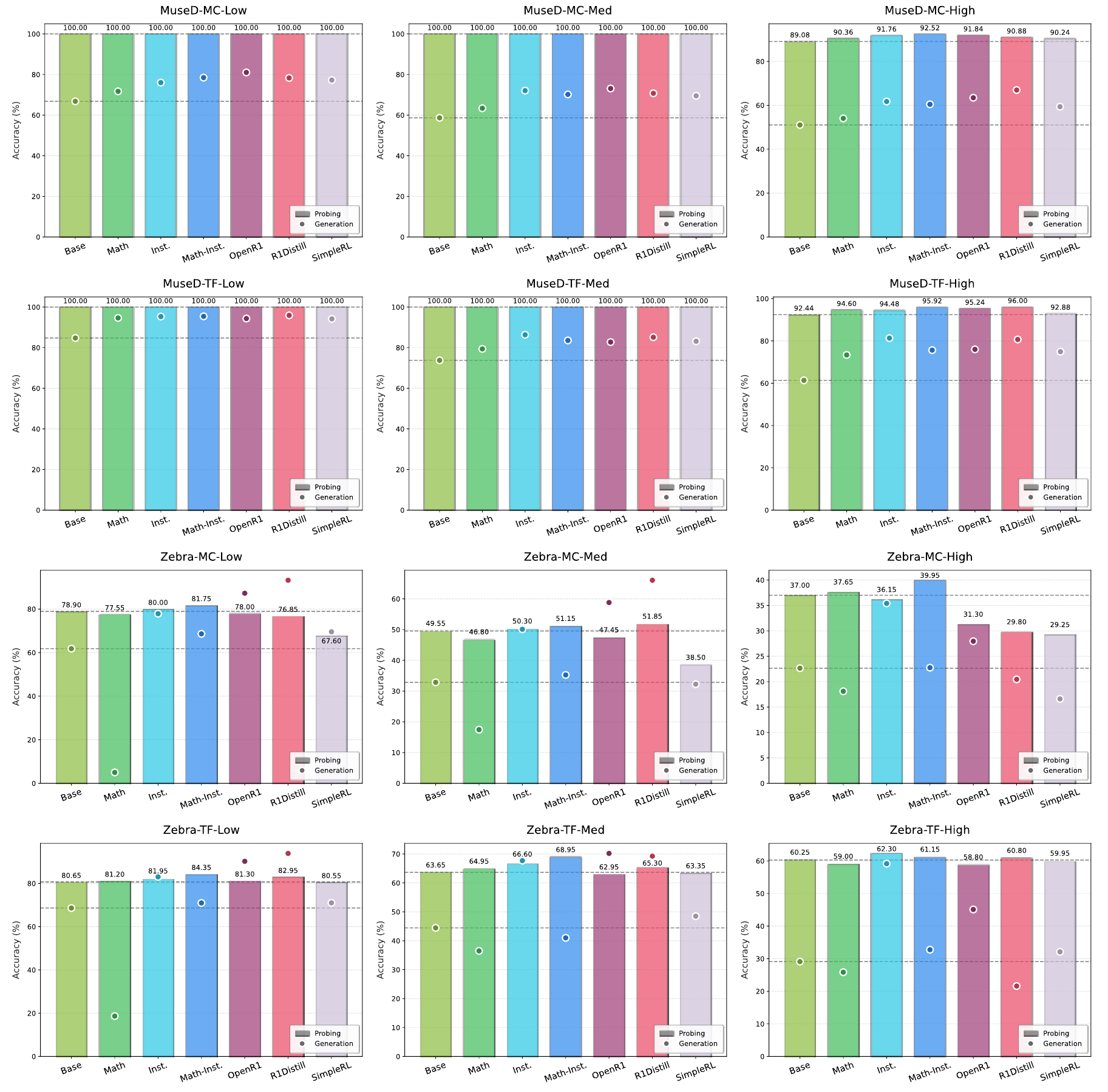}}
    \caption{
      \textbf{Qwen2.5-7B results (Part 1): Development of initial representation quality and generation accuracy across all tasks.} Gray dashed lines indicate the baseline performance of the Base model for probing and generation. 
    }
    \label{fig:exp1-qwen-7-all-1}
  \end{center}
\end{figure*}

\begin{figure*}[htbp]
  \begin{center}
    \centerline{\includegraphics[width=0.67\textwidth]{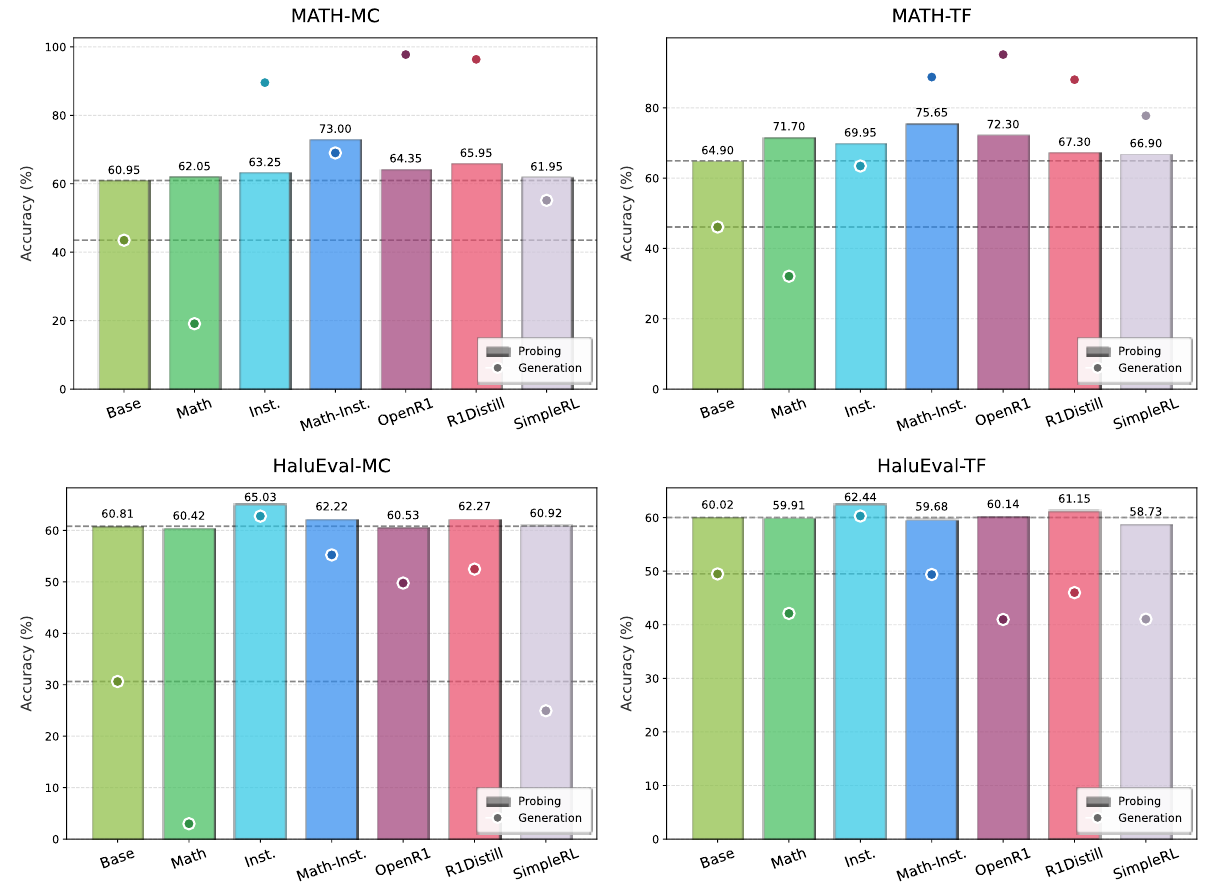}}
    \caption{
      \textbf{Qwen2.5-7B results (Part 2): Development of initial representation quality and generation accuracy across all tasks.} Gray dashed lines indicate the baseline performance of the Base model for probing and generation.
    }
    \label{fig:exp1-qwen-7-all-2}
  \end{center}
\end{figure*}

\begin{figure*}[htbp]
  \begin{center}
    \centerline{\includegraphics[width=\textwidth]{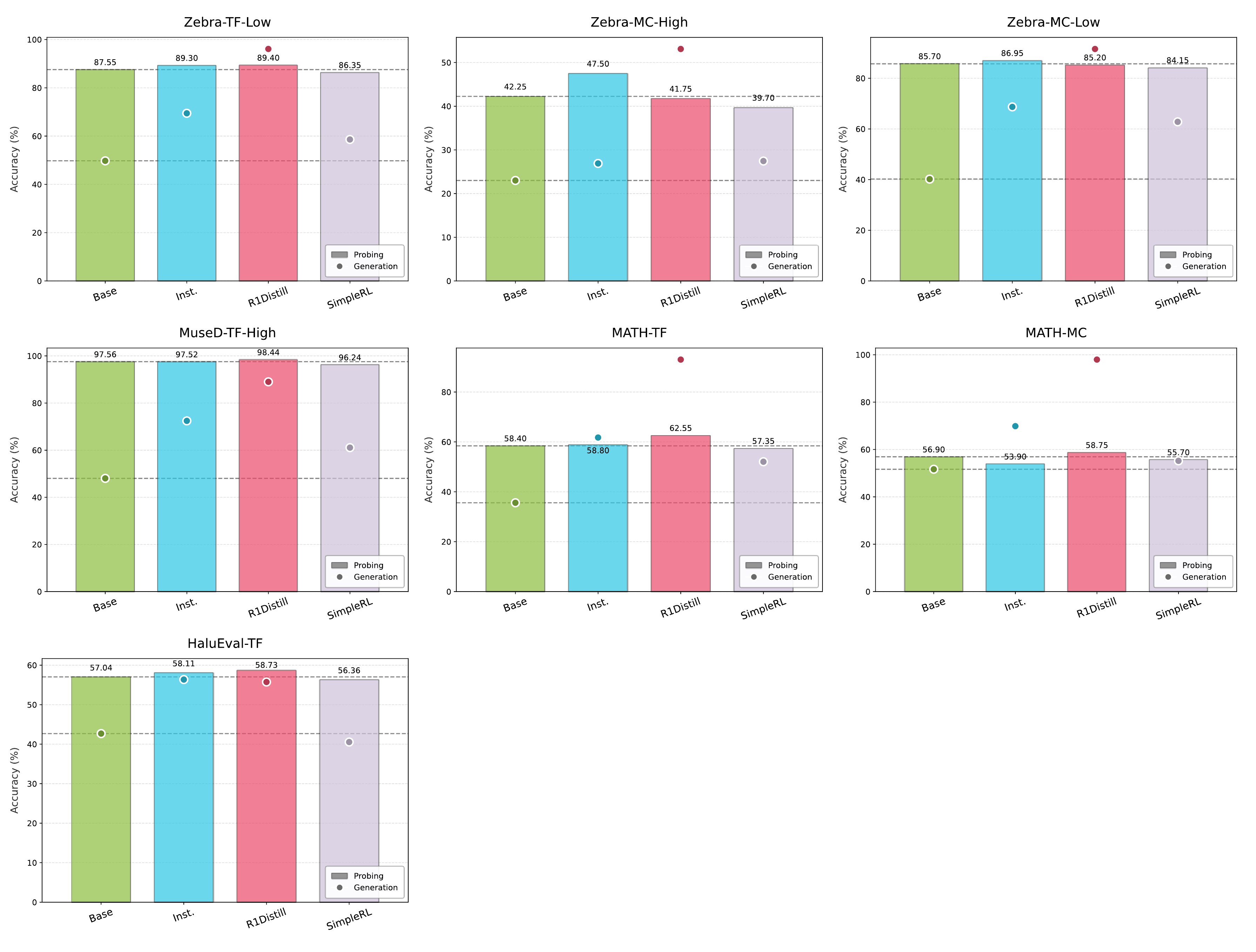}}
    \caption{
      \textbf{Llama3.1-8B results: Development of initial representation quality and generation accuracy on representative tasks.} Gray dashed lines indicate the baseline performance of the Base model for probing and generation.
    }
    \label{fig:exp1-llama-8-all}
  \end{center}
\end{figure*}

\begin{table}[htbp]
  \caption{\textbf{Llama3.1-8B results: Maximum performance gain} (Max $\Delta$) represents the highest accuracy improvement (\%) achieved by any other model over the Base model. The ratio indicates the relative scale of probing gains compared to generation gains.}
  \label{tab:llama delta comparison}
  \begin{center}
    \begin{small}
      \begin{sc}
      \resizebox{\columnwidth}{!}{
            \begin{tabular}{l|ccccccc}
            \toprule[1.5px]
            Max $\Delta$ & Zebra-TF-Low & Zebra-MC-High & Zebra-MC-Low & MuseD-TF-High  &  MATH-TF & MATH-MC & HaluEval-TF \\ \midrule
            Probing gain & 1.85 & 5.25 & 1.25 & 0.58 & 4.15 & 1.85 & 1.69 \\
            Generation gain & 46.35 & 30.10 & 51.35 & 40.96 & 57.40 & 46.30 & 13.68 \\
            \rowcolor{blue!10} \textbf{Ratio (Prob / Gen)} & \textbf{0.04} & \textbf{0.17} & \textbf{0.02} & \textbf{0.01} & \textbf{0.07} & \textbf{0.04} & \textbf{0.12} \\ 
            \bottomrule[1.5px]
            \end{tabular}
    }
    \end{sc}
    \end{small}
  \end{center}
\end{table}

\begin{figure*}[htbp]
  \begin{center}
    \centerline{\includegraphics[width=\textwidth]{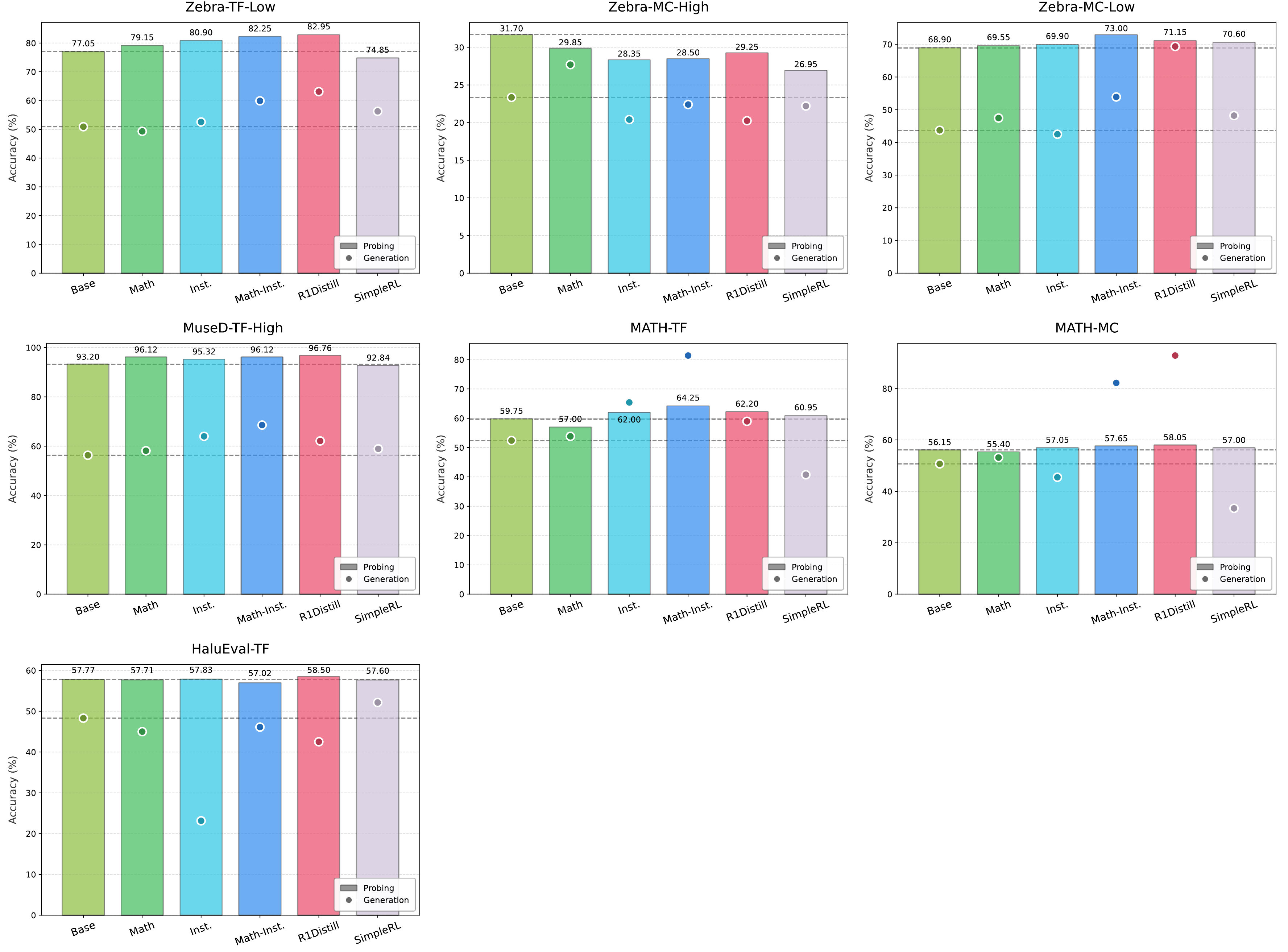}}
    \caption{
      \textbf{Qwen2.5-1.5B results: Development of initial representation quality and generation accuracy on representative tasks.} Gray dashed lines indicate the baseline performance of the Base model for probing and generation.
    }
    \label{fig:exp1-qwen-15-all}
  \end{center}
\end{figure*}

\begin{table}[htbp]
  \caption{\textbf{Qwen2.5-1.5B results: Maximum performance gain} (Max $\Delta$) represents the highest accuracy improvement (\%) achieved by any other model over the Base model. The ratio indicates the relative scale of probing gains compared to generation gains.}
  \label{tab:qwen 1.5 delta comparison}
  \begin{center}
    \begin{small}
      \begin{sc}
      \resizebox{\columnwidth}{!}{
            \begin{tabular}{l|ccccccc}
            \toprule[1.5px]
            Max $\Delta$ & Zebra-TF-Low & Zebra-MC-High & Zebra-MC-Low & MuseD-TF-High  &  MATH-TF & MATH-MC & HaluEval-TF \\ \midrule
            Probing gain & 5.90 & 0.00 & 4.10 & 3.56 & 4.50 & 1.90 & 0.73 \\
            Generation gain & 12.20 & 4.35 & 25.65 & 12.24 & 29.00 & 42.15 & 3.83 \\
            \rowcolor{blue!10} \textbf{Ratio (Prob / Gen)} & \textbf{0.48} & \textbf{0.00} & \textbf{0.16} & \textbf{0.29} & \textbf{0.16} & \textbf{0.05} & \textbf{0.19} \\ 
            \bottomrule[1.5px]
            \end{tabular}
    }
    \end{sc}
    \end{small}
  \end{center}
\end{table}

The additional results from the Llama3.1-8B and Qwen2.5-1.5B series demonstrate findings consistent with the main body. First, both model families confirm the existence of latent reasoning ability in the pre-trained models, with probing accuracies significantly above the majority guessing baseline. Second, the improvement in representation quality after post-training remains marginal, typically staying below $5\%$ across training stages. Notably, reasoning models even exhibit a degradation in representation quality after training on complex tasks (e.g., MATH and ZebraPuzzle) in Qwen2.5-7B series. Compared with maximum generation gains, the maximum improvement in probing is also extremely low, especially for Llama3.1-8B series. Nevertheless, probing accuracy exceeds generation accuracy in the majority of cases, except for highly optimized reasoning models. Finally, consistent with our main results, reasoning models trained via distillation consistently outperform the pure RL variants in both representation quality and generation accuracy.

\subsection{Representation Quality Dynamics during Generation}
\label{sec:exp2 more}
For this experiment, we select one high-difficulty subset for each task, as well as specific sub-tasks where the model's generation accuracy exceeds its initial probing accuracy. The complete results for the Qwen2.5-7B series are presented in \cref{fig:exp2-qwen-7}. 

\begin{figure*}[htbp]
  \begin{center}
    \centerline{\includegraphics[width=\textwidth]{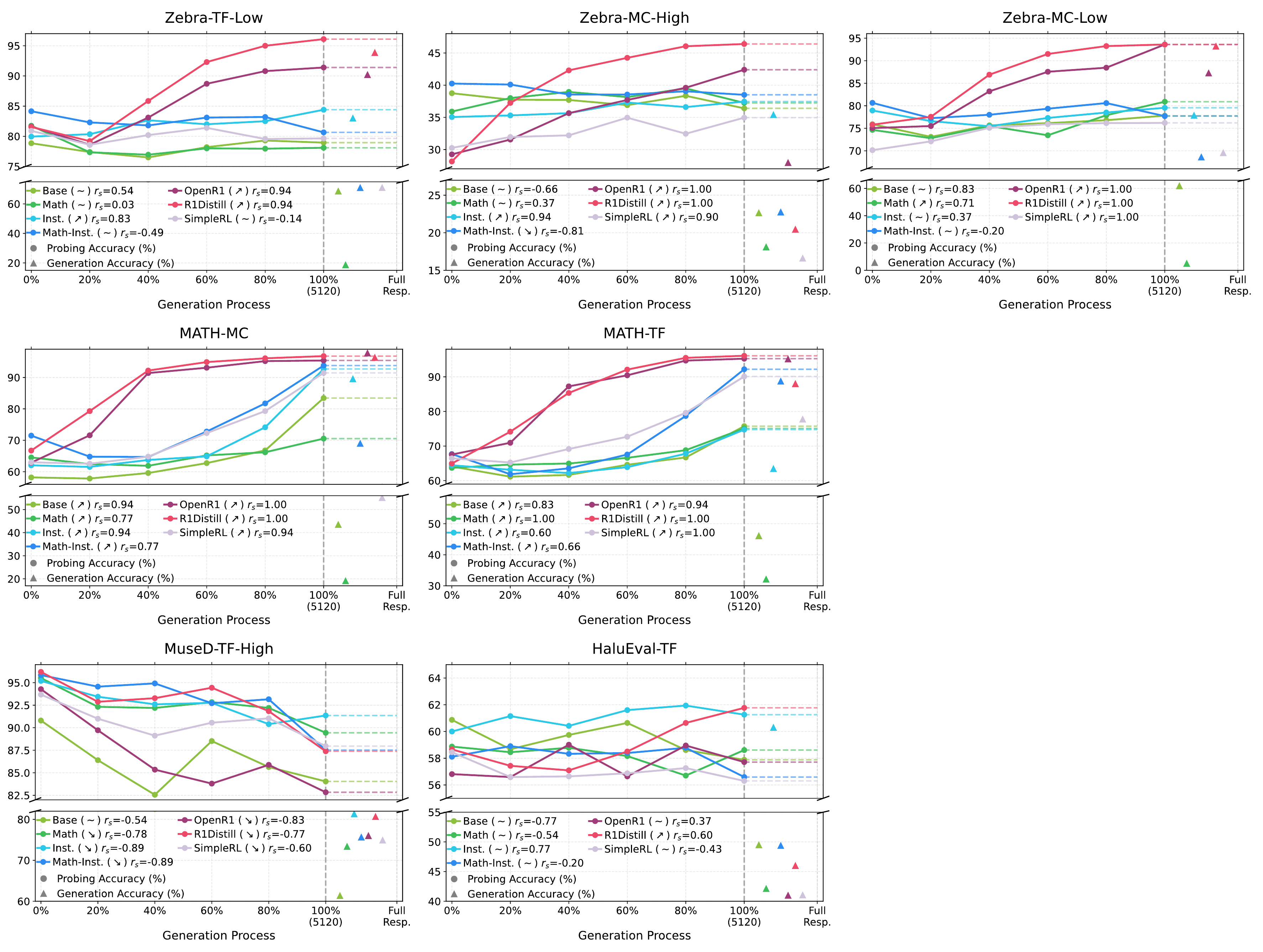}}
    \caption{
      \textbf{Qwen2.5-7B results: Representation quality dynamics during generation on all selected tasks.} Trends are analyzed using linear regression and Spearman rank correlation.
    }
    \label{fig:exp2-qwen-7}
  \end{center}
\end{figure*}

With the exception of the factual task (HaluEval) and the task where the model already possesses substantial initial representation quality (MuseD), strong reasoning models (OpenR1 and R1Distill) demonstrate a consistent, significant upward trend in representation quality during generation, ultimately achieving the highest values among all models. While non-reasoning models also show gradual improvement via CoT, the gain is markedly smaller compared to reasoning models. 

\begin{figure*}[htbp]
  \begin{center}
    \centerline{\includegraphics[width=\textwidth]{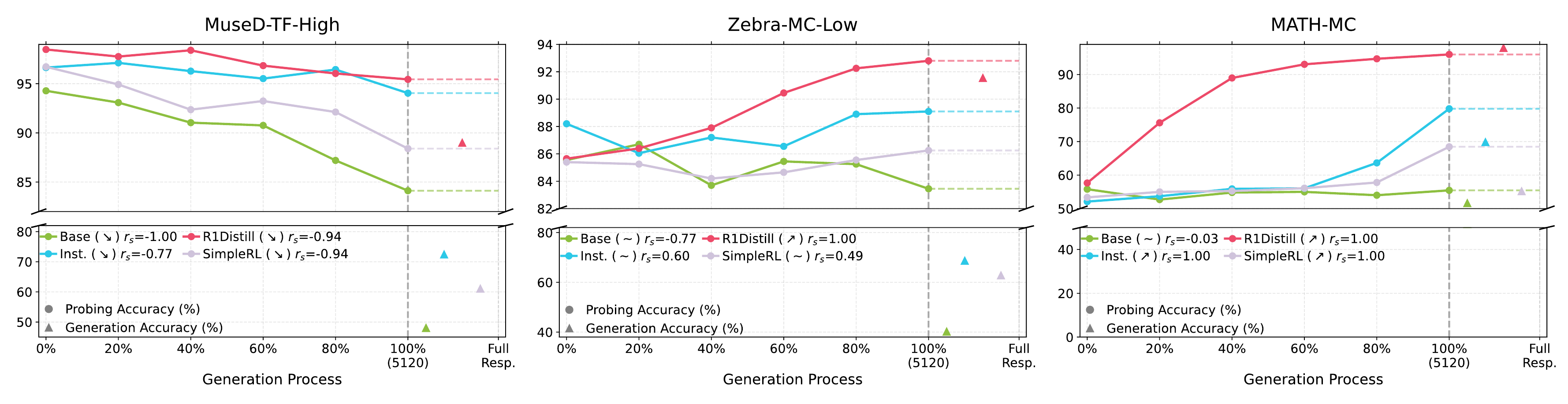}}
    \caption{
      \textbf{Llama3.1-8B results: Representation quality dynamics during generation on representative tasks.} Trends are analyzed using linear regression and Spearman rank correlation.
    }
    \label{fig:exp2-llama-8}
  \end{center}
\end{figure*}

\begin{figure*}[htbp]
  \begin{center}
    \centerline{\includegraphics[width=\textwidth]{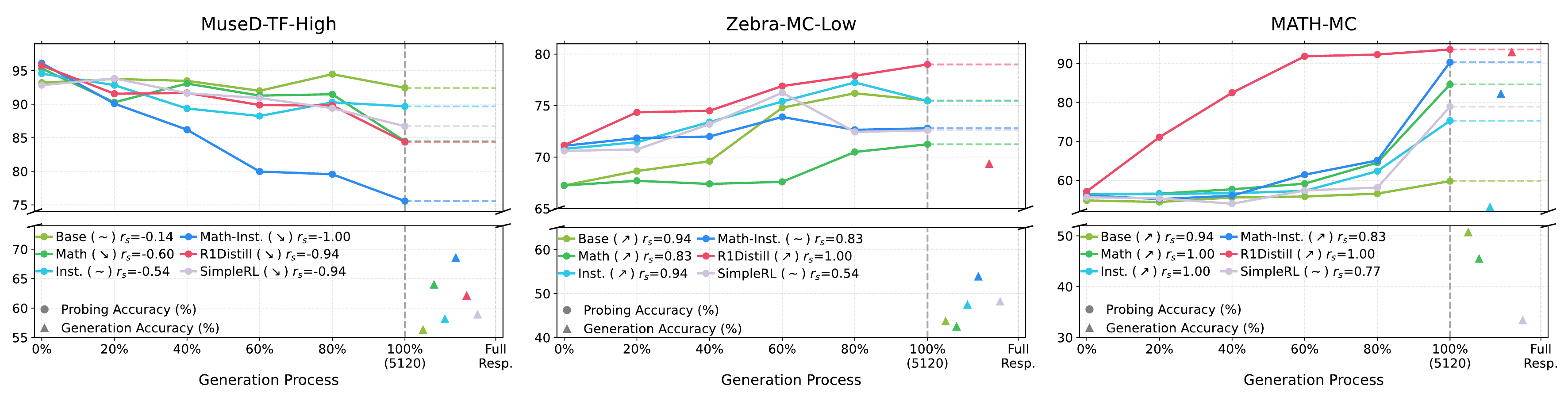}}
    \caption{
      \textbf{Qwen2.5-1.5B results: Representation quality dynamics during generation on representative tasks.} Trends are analyzed using linear regression and Spearman rank correlation.
    }
    \label{fig:exp2-qwen-1.5}
  \end{center}
\end{figure*}

We select three typical tasks where the Inst. model exhibits downward, fluctuating, and rising trends (MuseD-TF-High, Zebra-MC-Low, and MATH-MC, respectively), and replicate the experiments on the Llama3.1-8B series and Qwen2.5-1.5B series. As shown in \cref{fig:exp2-llama-8} and \cref{fig:exp2-qwen-1.5}, the observed trends remain consistent across different model series. Notably, R1Distill demonstrates the most significant upward trend and achieves the highest final representation quality on the latter two tasks.

\begin{figure*}[htbp]
  \begin{center}
    \centerline{\includegraphics[width=\textwidth]{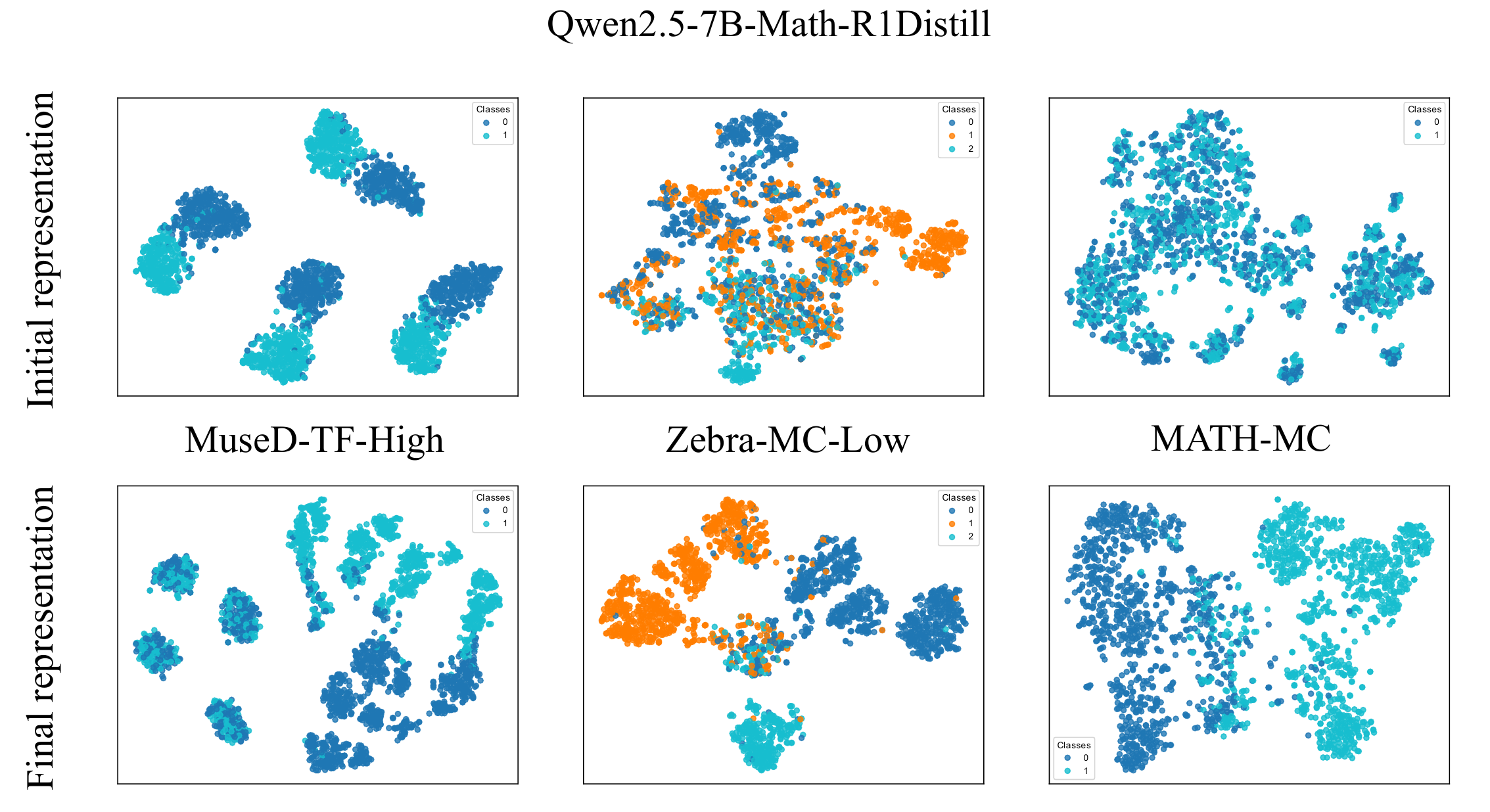}}
    \caption{
      \textbf{Visualization of Representation Shifts.} Comparison of the R1Distill model's representation distributions before (Initial) and after (Final) the CoT across three tasks.
    }
    \label{fig:repr change}
  \end{center}
\end{figure*}

To intuitively visualize how the model's internal representations change during generation, we compare the initial and final representation distributions of R1Distill using t-SNE \cite{maaten2008visualizing}. As shown in \cref{fig:repr change}, for MuseD, the initial representations already exhibit strong class separability, and the separation diminishes after CoT, indicating that the reasoning process can negatively impact internal judgment. In contrast, for Zebra and MATH, where initial representations show significant overlap between classes, the generation process yields a marked improvement in cluster separation, confirming that reasoning effectively enables the model to distinguish between different labels.

\begin{figure*}[htbp]
  \begin{center}
    \centerline{\includegraphics[width=\textwidth]{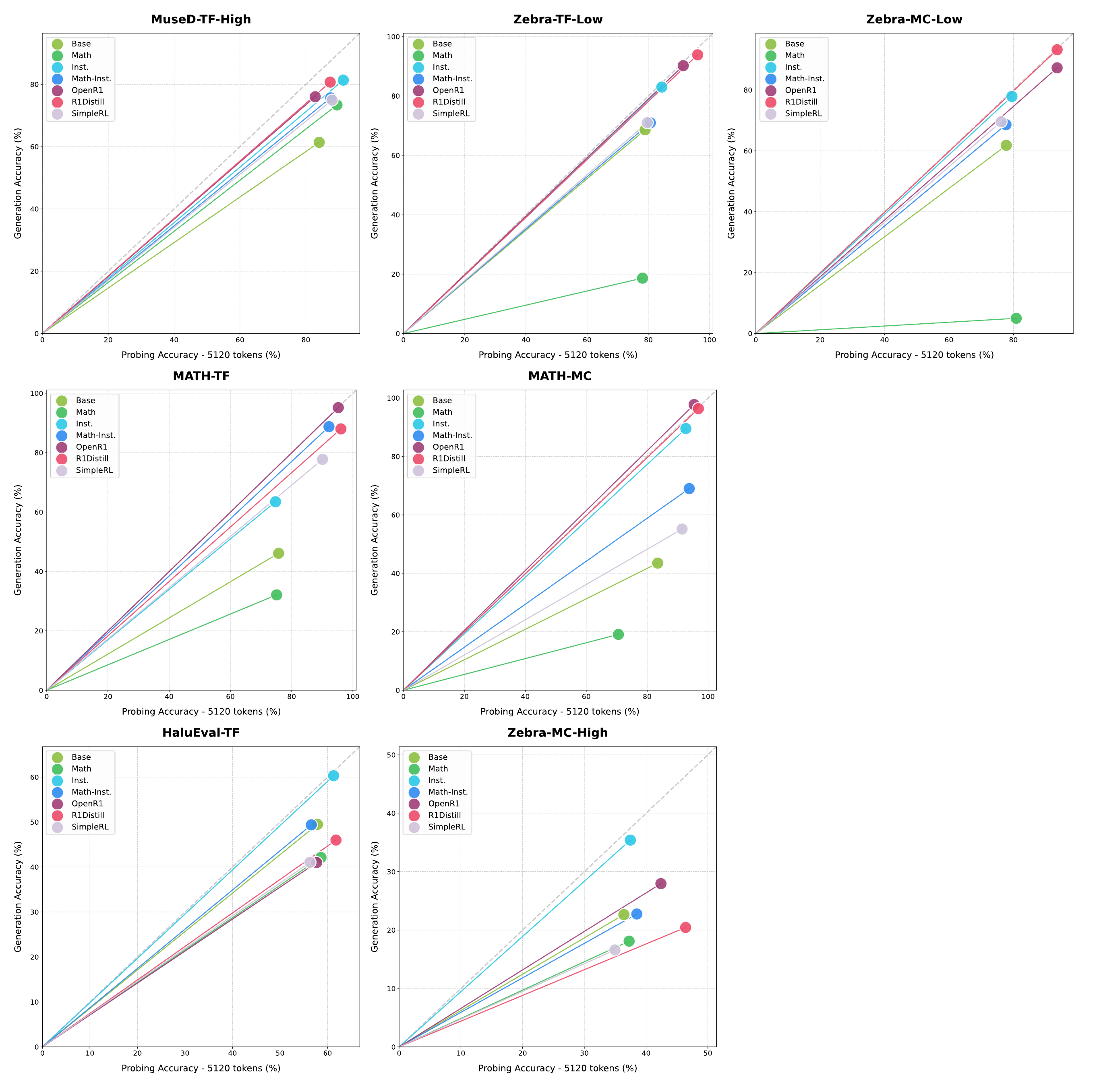}}
    \caption{
      \textbf{Qwen2.5-7B results: Generation accuracy versus probing accuracy on all selected tasks.}
    }
    \label{fig:realization-qwen-7}
  \end{center}
\end{figure*}

\begin{figure*}[htbp]
  \begin{center}
    \centerline{\includegraphics[width=\textwidth]{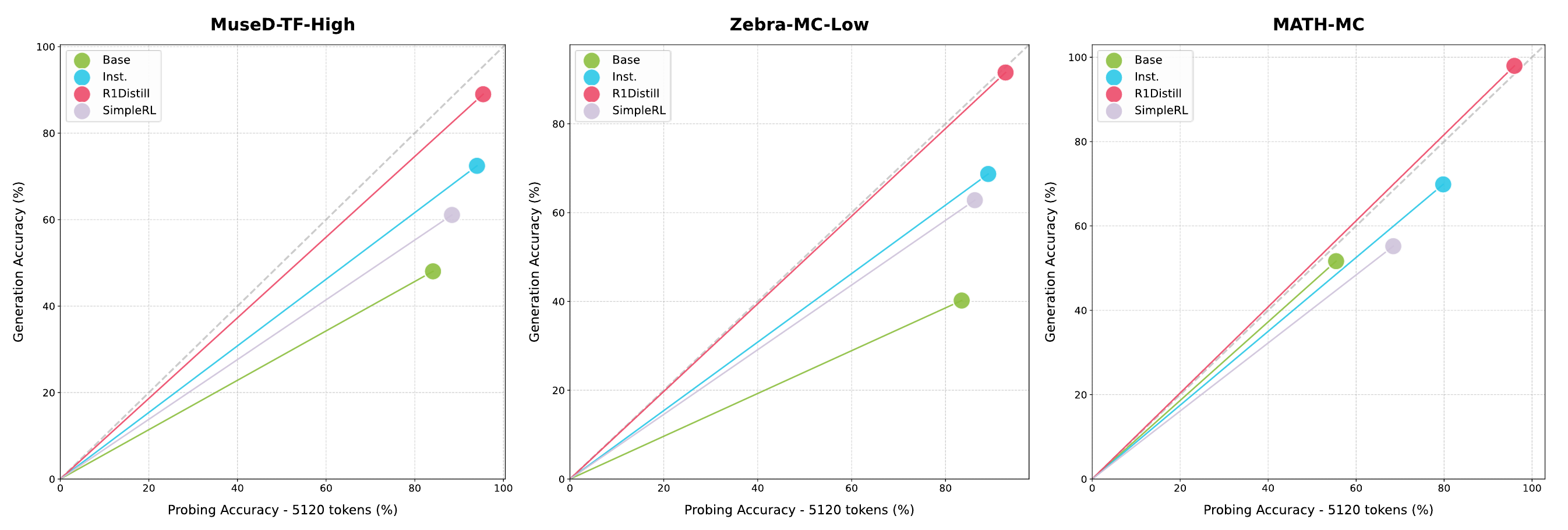}}
    \caption{
      \textbf{Llama3.1-8B results: Generation accuracy versus probing accuracy on representative tasks.}
    }
    \label{fig:realization-llama-8}
  \end{center}
\end{figure*}

\begin{figure*}[htbp]
  \begin{center}
    \centerline{\includegraphics[width=\textwidth]{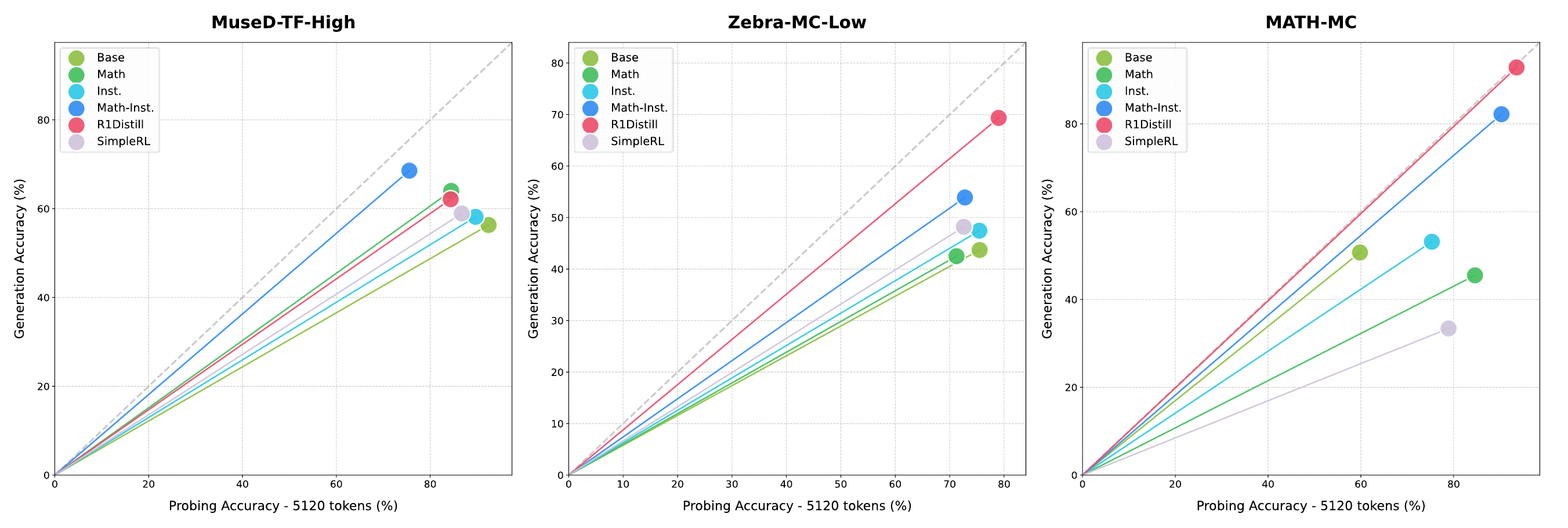}}
    \caption{
      \textbf{Qwen2.5-1.5B results: Generation accuracy versus probing accuracy on representative tasks.}
    }
    \label{fig:realization-qwen-1.5}
  \end{center}
\end{figure*}

Finally, we compare explicit generation accuracy with latent probing accuracy, as shown in \cref{fig:realization-qwen-7}, \cref{fig:realization-llama-8}, and \cref{fig:realization-qwen-1.5}. Across all models and tasks, generation accuracy consistently fails to exceed or match probing accuracy, even though probing is restricted to the first 5,120 tokens rather than the full response. This indicates that the model's latent potential remains underutilized. Post-training generally enhances the model's ability to realize this potential, as evidenced by steeper slopes compared to pre-trained models. 

\subsection{Statistical Analysis of Alignment}
\label{sec:exp3 more}
In this section, we provide the full statistical results for all reasoning tasks (High-difficulty level) in \cref{tab:repr-gen alignment all} and \cref{tab:repr-repr alignment all}. Additionally, we visualize the relationship between the initial and final probing probabilities $p$ in \cref{fig:exp3-inst} and \cref{fig:exp3-distill}. When calculating the trend based on linear regression and the Spearman rank correlation coefficient $r_s$, we exclude buckets containing fewer than $10$ data samples to ensure statistical robustness.

\begin{table*}[htbp]
  \caption{
  \textbf{Generation-representation alignment.} Significance levels $P_{MWU}$ of Mann-Whitney U test: NS denotes non-significant ($>5e^{-2}$), * ($<5e^{-2}$), and *** ($<1e^{-10}$). \textbf{Bold}: higher alignment with CoT; \textcolor{blue}{Blue}: R1Distill alignment is lower than Inst.}
  \label{tab:repr-gen alignment all}
  \begin{center}
    \begin{small}
      \begin{sc}
      \resizebox{0.9\textwidth}{!}{
            \begin{tabular}{l|cccc|cccc}
            \toprule[1.5px]
            & \multicolumn{4}{c|}{Inst.} & \multicolumn{4}{c}{R1Distill} \\
            High-Difficulty Task &   Trend    &   $r_s$$\uparrow$   &   ROC-AUC$\uparrow$   &   $P_{MWU}$   &   Trend    &   $r_s$$\uparrow$   &   ROC-AUC$\uparrow$   &   $P_{MWU}$ \\ \midrule
            MuseD-TF        &   $\nearrow$  &   0.59  &  	0.57    &   *   &   $\sim$    &   \textcolor{blue}{0.07}    &  \textcolor{blue}{0.53}     &    ns   \\ \midrule
            \textbf{MuseD-TF+CoT}    &  $\nearrow$   &   \textbf{0.99}   &    \textbf{0.85}  &  ***    &   $\nearrow$    &   \textcolor{blue}{\textbf{0.95}}    &  \textbf{0.88}     &    ***   \\ \midrule
            MuseD-MC  & $\sim$ & 0.49 & 0.57 & * & $\sim$ & \textcolor{blue}{0.28} & 0.57 & * \\ \midrule
            \textbf{MuseD-MC+CoT} & $\nearrow$ & \textbf{0.90} & \textbf{0.78} & *** & $\nearrow$ & \textbf{0.96} & \textcolor{blue}{\textbf{0.78}} & *** \\ \midrule
            Zebra-TF & $\nearrow$ & 0.88 & 0.55 & * & $\sim$ & \textcolor{blue}{0.07} & \textcolor{blue}{0.51} & ns \\ \midrule
            \textbf{Zebra-TF+CoT} & $\nearrow$ & \textbf{1.00} & \textbf{0.86} & *** & $\nearrow$ & \textcolor{blue}{\textbf{0.76}} & \textcolor{blue}{\textbf{0.62}} & *** \\ \midrule
            Zebra-MC        &   $\nearrow$   &  1.00    &   0.61   &  ***    &     $\nearrow$  &   1.00    &   \textcolor{blue}{0.60}    &    *   \\ \midrule
            \textbf{Zebra-MC+CoT}    &   $\nearrow$   &  0.99    &    \textbf{0.94}  &   ***   &   $\nearrow$    &   \textcolor{blue}{0.98}    &   \textcolor{blue}{\textbf{0.76}}    &   ***   \\ \midrule
            HaluEval-TF     &   $\sim$   &   0.45   &   0.56   &   *   &   $\sim$    &   \textcolor{blue}{0.21}    &    \textcolor{blue}{0.50}   &    ns   \\ \midrule
            MATH-TF         &   $\nearrow$   &   0.88   &   0.59   &   *   &  $\sim$    &   \textcolor{blue}{0.45}    &    \textcolor{blue}{0.57} &      * \\ \midrule
            \textbf{MATH-TF+CoT}     &   $\nearrow$   &   \textbf{0.98}   &  \textbf{0.70}   &  ***    &   $\nearrow$    &   \textcolor{blue}{\textbf{0.52}}    &   \textbf{0.90}    &   ***   \\ \midrule
            MATH-MC & $\sim$ & -0.16 & 0.50 & ns & $\searrow$ & \textcolor{blue}{-0.60} & \textcolor{blue}{0.46} & ns \\ \midrule
            \textbf{MATH-MC+CoT} & $\nearrow$ & \textbf{0.95} & \textbf{0.95} & *** & $\sim$ & \textcolor{blue}{\textbf{0.43}} & \textcolor{blue}{\textbf{0.85}} & *** \\
            \bottomrule[1.5px]
            \end{tabular}
    }
      \end{sc}
    \end{small}
  \end{center}
\end{table*}

\begin{table}[htbp]
  \caption{\textbf{Alignment between representations before and after CoT}. $r_s$: Spearman rank correlation coefficient; $r_p$: Pearson correlation coefficient; $r^2$ is adopted as the primary metric of linear regression, as $p$-values tend to become uninformative in large datasets. \textcolor{blue}{Blue}: R1Distill alignment is lower than Inst.}
  \label{tab:repr-repr alignment all}
  \begin{center}
    \begin{small}
      \begin{sc}
      \resizebox{0.5\columnwidth}{!}{
            \begin{tabular}{l|ccc|ccc}
            \toprule[1.5px]
            & \multicolumn{3}{c|}{Inst.} & \multicolumn{3}{c}{R1Distill} \\
            &   $r_s$$\uparrow$    &   $r_p$$\uparrow$   &   $r^2$$\uparrow$   &   $r_s$$\uparrow$    &   $r_p$$\uparrow$   &   $r^2$$\uparrow$ \\ \midrule
            MuseD-TF-High & 0.16 & 0.13 & 0.02 & \textcolor{blue}{0.10} & \textcolor{blue}{0.05} & \textcolor{blue}{0.00}  \\ \midrule
            MuseD-MC-High & 0.36 & 0.30 & 0.09 & \textcolor{blue}{0.25} & \textcolor{blue}{0.21} & \textcolor{blue}{0.04} \\ \midrule
            Zebra-TF-High & 0.30 & 0.30 & 0.09 & \textcolor{blue}{0.29} & \textcolor{blue}{0.29} & \textcolor{blue}{0.08} \\ \midrule
            Zebra-MC-High & 0.23 & 0.32 & 0.10 & \textcolor{blue}{0.18} & \textcolor{blue}{0.15} & \textcolor{blue}{0.02} \\ \midrule
            MATH-TF & 0.28 & 0.29 & 0.08 & \textcolor{blue}{0.09} & \textcolor{blue}{0.05} & \textcolor{blue}{0.00} \\ \midrule
            MATH-MC & 0.13 & 0.07 & 0.01 & \textcolor{blue}{0.08} & \textcolor{blue}{0.01} & \textcolor{blue}{0.00} \\
            \bottomrule[1.5px]
            \end{tabular}
    }
    \end{sc}
    \end{small}
  \end{center}
\end{table}

Consistent with the results in section \ref{sec:more exp1}, generation correctness exhibits minimal correlation with the initial representation but significantly higher correlation with the final representation. Notably, R1Distill generally shows weaker alignment than Inst. Furthermore, the correlation between initial and final representations is low for both models, with R1Distill demonstrating an even lower degree of alignment.

\begin{figure*}[htbp]
  \begin{center}
    \centerline{\includegraphics[width=\textwidth]{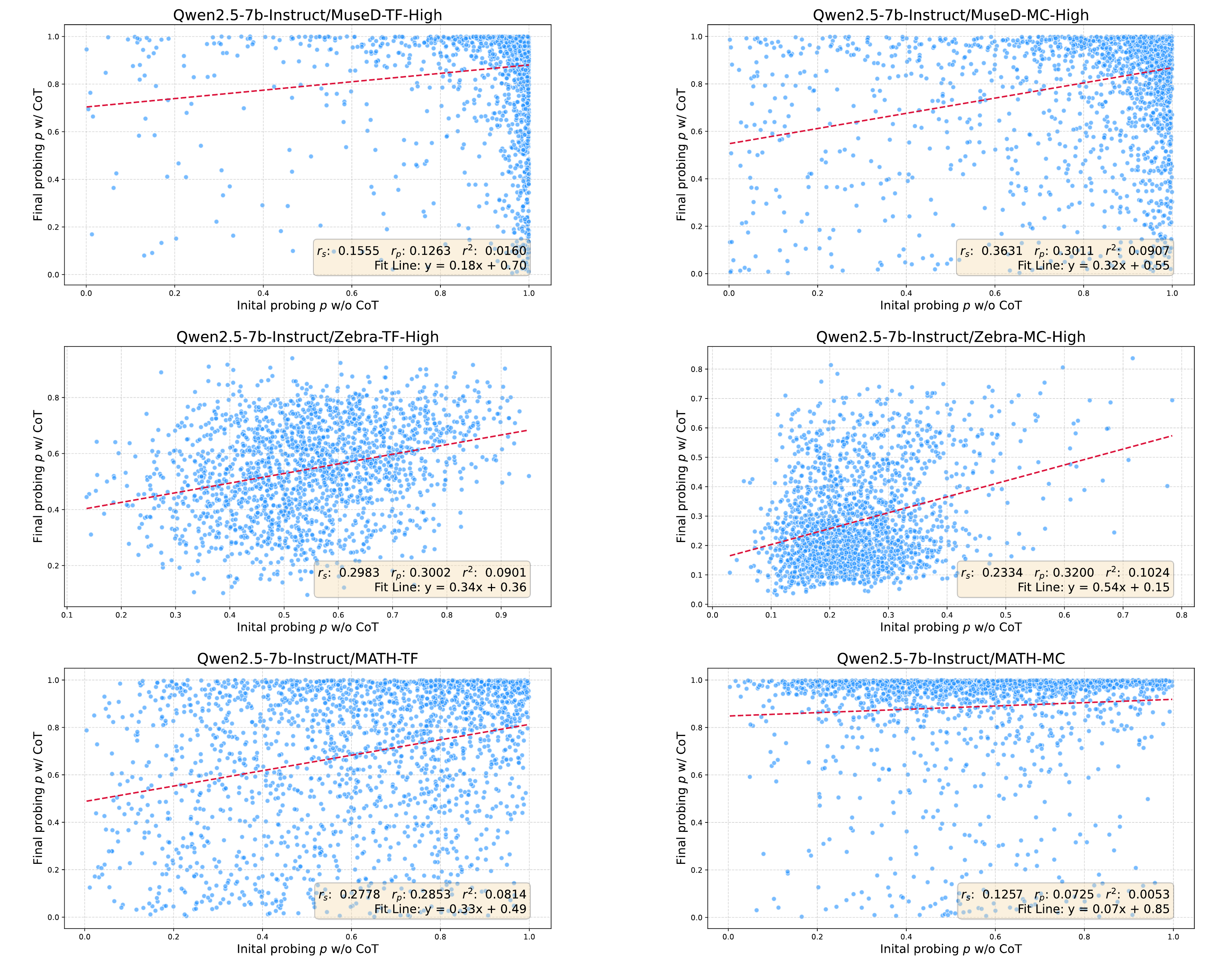}}
    \caption{
      \textbf{Qwen2.5-7B-Instruct: Visualization of the relationship between initial and final probing probabilities $p$.}
    }
    \label{fig:exp3-inst}
  \end{center}
\end{figure*}

\begin{figure*}[htbp]
  \begin{center}
    \centerline{\includegraphics[width=\textwidth]{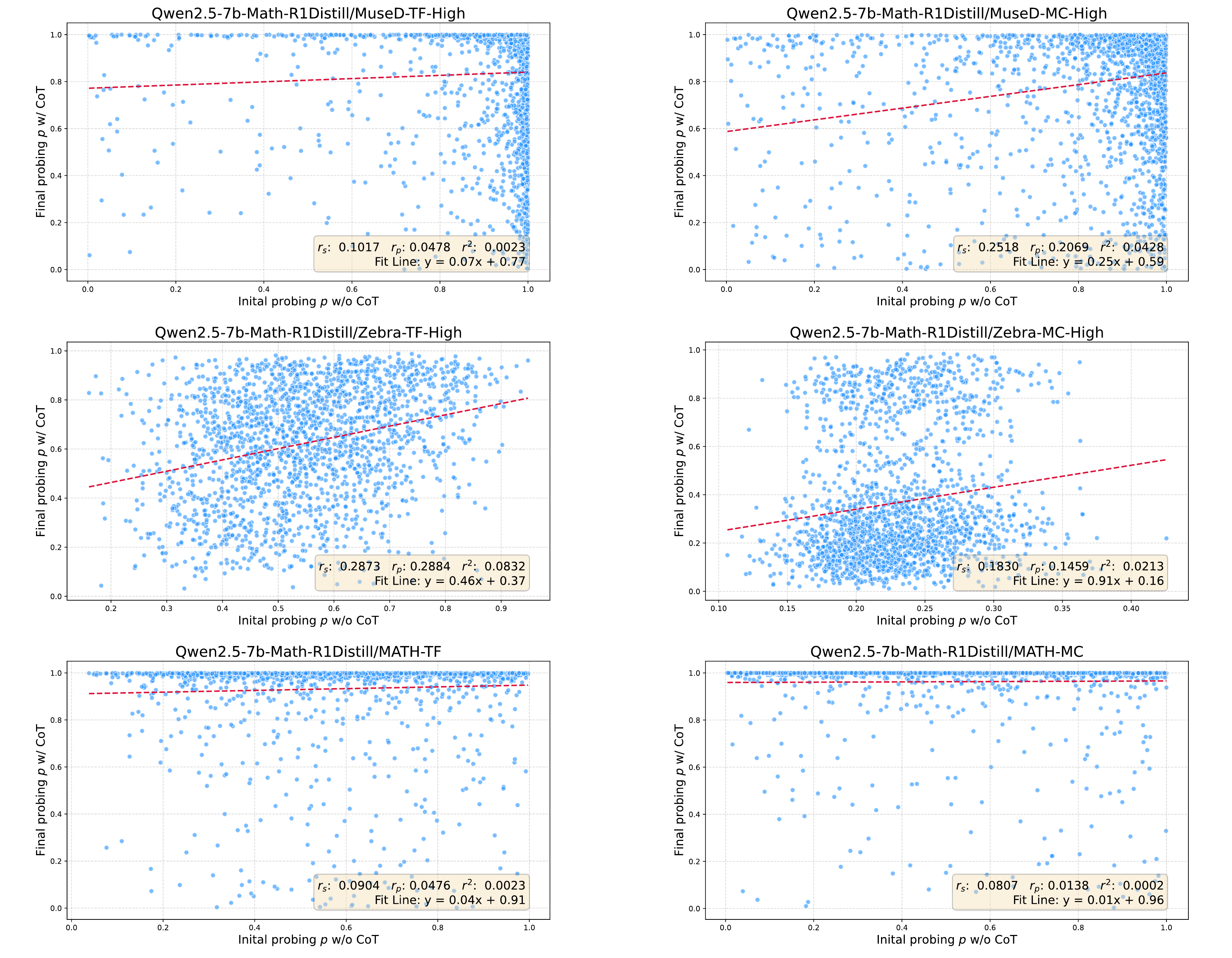}}
    \caption{
      \textbf{Qwen2.5-7B-Math-R1Distill: Visualization of the relationship between initial and final probing probabilities $p$.}
    }
    \label{fig:exp3-distill}
  \end{center}
\end{figure*}

The visualizations in \cref{fig:exp3-inst} and \cref{fig:exp3-distill} further demonstrate the drastic changes occurring between the initial and final representations.

\clearpage




\end{document}